\documentclass[a4paper,12pt]{article}

\usepackage[margin=25mm]{geometry}
 \usepackage{graphicx}
 \usepackage{xcolor}
 \usepackage{verbatim}
 \usepackage{multicol,multirow}
 \usepackage[utf8]{inputenc}
 \usepackage{amsmath,amsthm} \allowdisplaybreaks
 \usepackage{amssymb}
 \usepackage{array}
 \usepackage{booktabs}
\usepackage{caption}
\usepackage{mathtools}
\usepackage{csvsimple}
\usepackage{subcaption}
\usepackage{hyperref}
\newtheoremstyle{named}{}{}{\itshape}{}{\bfseries}{.}{.5em}{\thmnote{#3}}
\theoremstyle{named}
\newtheorem*{namedtheorem}{Theorem}

\newcommand{\NP}{$\mathcal{NP}$}

\providecommand{\keywords}[1]
{
  \small	
  \textbf{\textit{Keywords---}} #1
}

\newtheorem{theorem}{Theorem}%  meant for continuous numbers
\title%[]
{Exact methods and lower bounds for the \\Oven Scheduling Problem}

\date{February 2022}%

\author{Marie-Louise Lackner$^{1}$, 
Christoph Mrkvicka$^{2}$,
Nysret Musliu$^{1}$, \\
Daniel Walkiewicz$^{2}$,
Felix Winter$^{1}$\\
\\
\small \texttt{marie-louise.lackner@tuwien.ac.at}\\
\small \texttt{christoph.mrkvicka@mcp-alfa.com}\\
\small \texttt{nysret.musliu@tuwien.ac.at}\\
\small \texttt{daniel.walkiewicz@mcp-alfa.com}\\
\small \texttt{felix.winter@tuwien.ac.at}\\
\\
 $^{1}$\small\emph{Christian Doppler Laboratory for Artificial Intelligence and Optimization} \\
 \small\emph{for Planning and Scheduling}, DBAI, TU Wien, Vienna,  Austria \\
 $^{2}$\small \emph{MCP GmbH}, Vienna, Austria \\
}

\begin{document}
\maketitle

\abstract{The Oven Scheduling Problem (OSP) is a new parallel batch scheduling problem that arises in the area of electronic component manufacturing. Jobs need to be scheduled to one of several ovens and may be processed simultaneously in one batch if they have compatible requirements.
The scheduling of jobs must respect several constraints concerning eligibility and availability of ovens, release dates of jobs, setup times between batches as well as oven capacities.
Running the ovens is highly energy-intensive and thus the main objective, besides finishing jobs on time, is to minimize the cumulative batch processing time across all ovens.
This objective distinguishes the OSP from other batch processing problems which typically minimize objectives related to makespan, tardiness or lateness.

We propose to solve this NP-hard scheduling problem via constraint programming (CP) and integer linear programming (ILP) and
present corresponding models. 
For an experimental evaluation,
we introduce a multi-parameter random instance generator to provide a diverse set of problem instances. 
Using state-of-the-art solvers, we evaluate  
the quality and compare the performance of our CP- and ILP-models.
We show that our models can find feasible solutions for instances of realistic size, many of those being provably optimal or nearly optimal solutions. 
Finally, we derive theoretical lower bounds on the solution cost of feasible solutions to the OSP; these can be computed within a few seconds. We show that these lower bounds are competitive with those derived by state-of-the-art solvers.
}

\hspace{10pt}


\maketitle

\section{Introduction}
\label{sec:intro}
In the electronics industry, many components need to undergo a hardening process which is performed in specialised heat treatment ovens.
As running these ovens is a highly energy-intensive task, it is advantageous to group multiple jobs that produce compatible components into batches for simultaneous processing. 
However, creating an efficient oven schedule is a complex task as several cost objectives related to oven processing time, job tardiness and setup costs  need to be minimized.
Furthermore, a multitude of constraints that impose restrictions on the availability, capacity, and eligibility of ovens have to be considered.
Due to the inherent complexity of the problem and the large number of jobs that usually have to be batched in real-life scheduling scenarios, efficient automated solution methods are thus needed to find optimized schedules.

Over the last three decades, a wealth of scientific papers %has been devoted to the study 
investigated batch scheduling problems.
Several early problem variants using single machine and parallel machine settings were categorized and shown to be NP-hard~\cite{potts_scheduling_2000}; a more recent literature review has been provided by Mathirajan et al.~\cite{mathirajan2006literature}.
Batch scheduling problems share the common goal that jobs are processed simultaneously in batches in order to increase efficiency. 
Besides this common goal, a variety of different problems with unique constraints and solution objectives arise from different applications in the 
chemical, aeronautical, electronic and steel-producing industry where batch processing machines can appear in the form of autoclaves~\cite{malapert2012constraint}, ovens~\cite{lee1992efficient} or kilns~\cite{zhao_decomposition_2020}.

For example, a just-in-time batch scheduling problem that aims to minimize tardiness and earliness objectives has been recently investigated in~\cite{polyakovskiy_just_time_2020}.
%and was proven to be NP-hard in~\cite{hazir_batch_2014}.
%Comment: recently investigates in 2020 and NP-harness in 2014. Is this right?
Another recent study~\cite{zhao_decomposition_2020} introduced a batch scheduling problem from the steel industry that includes setup times, release times, as well as due date constraints.
Furthermore, a complex two-phase batch scheduling problem from the composites manufacturing industry has been solved with the use of CP and hybrid techniques \cite{tang_cp_2020}.

Exact methods used for finding optimal schedules on batch processing machines involve dynamic programming~\cite{brucker1998scheduling} for the simplest variants as well as CP- and mixed integer programming (MIP) models. 
CP-models have e.g. been proposed by Malapert et al.~\cite{malapert2012constraint} and by Kosch et al.~\cite{kosch2014new}, where both publications consider batch scheduling on a single machine with non-identical job sizes and due dates but without release dates.
A novel arc-flow based CP-model for minimizing makespan on parallel batch processing machines was recently proposed~\cite{trindade2020arc}.
Branch-and-Bound~\cite{azizoglu2001scheduling} and Branch-and-Price~\cite{parsa2010branch} methods have been investigated as well. 
As the majority of batch scheduling problems are \NP-hard, exact methods are often not capable of solving large instances within a reasonable time-limit and thus (meta-)heuristic techniques are designed in addition. These range from GRASP approaches~\cite{damodaran2011grasp} and variable neighbourhood search~\cite{cakici2013batch}, over genetic algorithms~\cite{malve2007genetic, costa2014novel}, ant colony optimization~\cite{cheng2013improved} and particle swarm optimization~\cite{zhou2018modified} to simulated annealing~\cite{damodaran2012simulated}.

In this paper, we introduce the Oven Scheduling Problem (OSP), which is a new real-life batch scheduling problem from the area of electronic component manufacturing.
The OSP defines a unique combination of cumulative batch processing time, tardiness and setup cost objectives that need to be minimized.
To the best of our knowledge, this objective has not been studied previously in batch scheduling problems.
Furthermore, we take special requirements of the electronic component manufacturing industry into account.
Thus, the problem 
considers specialized constraints concerning the availability of ovens as well as constraints regarding oven capacity, oven eligibility and job compatibility. 

The main contributions of this paper are: 
\begin{itemize}
    \item We introduce and formally specify a new real-life batch scheduling problem. % from the electronics industry.
    \item Based on two different modelling approaches, we propose solver independent CP- and ILP-models that can be utilized with state-of-the-art solver technology to provide an exact solution approach. In addition, we provide two OPL-models for CP Optimizer using interval variables.
    \item To generate a large instance set, we introduce an innovative multi-parameter random instance generation procedure.
    \item We provide a construction heuristic that can be used to quickly obtain feasible solutions.
    \item We develop theoretical lower bounds on the number of batches required in a feasible solution as well as lower bounds on the value of the objective function.
    \item All our solution methods are extensively evaluated through a series of benchmark experiments, including the evaluation of several search strategies and a warm-start approach. For a sample of 80 benchmark instances, we obtain optimal results for 38 instances, and provide upper and lower bounds on the objective for all instances.
\end{itemize}

The current paper is a significant extension of our CP 2021 conference paper~\cite{lackner_et_al:LIPIcs.CP.2021.37}. One major addition to our conference paper is an entirely new modeling approach based on representative jobs for batches, which produces the best results on our benchmark set  (see Sections~\ref{sec:model-repr-job} 
and \ref{sec:cp-repr}).
Another addition is the derivation of refined lower bounds on the solution cost for a given instance of the OSP (in Section~\ref{sec:lower-bounds}).

The rest of this paper is organized as follows:
We first provide a description of the OSP (Section~\ref{sec:problem-description}) before we formally define the problem and formulate a CP model (Section~\ref{sec:cp-model}). Then we present an ILP model (Section~\ref{sec:mip-model}) and a second CP modelling approach using representative jobs (Section~\ref{sec:model-repr-job}). Alternative model implementations as well as search strategies are described in Section~\ref{sec:alternative-implementations+search}.
Theoretical lower bounds are derived in Section~\ref{sec:lower-bounds}.
In Section~\ref{sec:random+heuristic}, we describe the random instance generator and the construction heuristic.
Finally, we present and discuss experimental results (Section~\ref{sec:experiments}).

\section{Description of the OSP}
\label{sec:problem-description}

%Certain components produced in the electronics industry need to undergo a hardening process which is performed in specialised heat treatment ovens. Running these ovens is highly energy-intensive and thus it is advantageous to process compatible components simultaneously in a batch. In the following, we use more general terminology and speak of jobs that need to be processed in batches on machines.
The OSP consists in creating a feasible assignment of jobs to batches and in finding an optimal schedule of these batches on a set of ovens, which we refer to as machines in the remainder of the paper.

%Batches cannot be formed in an arbitrary fashion.
Jobs that are assigned to the same batch need to have the same \emph{attribute};\footnote{In the literature, the concept of attribute compatibility is often treated under the term of \emph{incompatible job families}, see, e.g.~\cite{cakici2013batch}.} %
in the context of heat treatment this can be thought of as the temperature at which components need to be processed.
Moreover, a batch cannot start before the \emph{release date} of any job assigned to this batch.
The \emph{batch processing time} may not be shorter than the minimal processing time of any assigned job and must not be longer then any job's maximal processing time, as this could damage the produced components.
Every job can only be assigned to a set of \emph{eligible machines} and machines are further only available during machine-dependent \emph{availability intervals}.
Moreover, machines have a maximal \emph{capacity}, which may not be exceeded by the cumulative size of jobs in a single batch.
\emph{Setup times} between consecutive batches must also be taken into account. Setup times depend on the ordered pair of attributes of the jobs in the respective batches and are independent of the machine assignments.
In the context of heat treatment, this can be thought of as the time required to switch from one temperature to another.
Moreover, setup times before the first batch on every machine need to be taken into account. Setup times before first batches depend on the \emph{initial state} of machines.\footnote{Note that initial states of machines were not yet present in the original formulation of the OSP  in the conference paper~\cite{lackner_et_al:LIPIcs.CP.2021.37}.}
%These necessary setup periods also incur attribute-dependent \emph{setup costs}, as e.g.\ cooling down is less expensive than heating up.

The objective of the OSP is to minimize the cumulative batch processing time, total setup costs, as well as the number of tardy jobs.
These three objective components are combined in a single objective function using a linear combination, as is formalized in Section~\ref{sec:cp-model:obj}.
As the minimization of job tardiness usually has the highest priority in practice, the tardiness objective has a higher weight than the other objectives.

In practice, the cumulative batch processing time should be minimized as the cost of running an oven depends merely on the processing time of the entire batch and not on the number of jobs within a batch. Therefore, running an oven containing a single small job incurs the same costs as running the oven filled to its maximal capacity.

Furthermore, we note that setup costs and setup times are not necessarily correlated. In fact, cooling down an oven from a high to a low temperature might not incur any (energy) costs, but still might require a certain amount of time. Setup costs are used to capture all costs that are related to the setup required between batches; e.g. costs related to personnel involved in the setup operation can also be captured by setup costs. Therefore, setup times are not included in the objective function, only setup costs are.\footnote{This is another difference to the original formulation of the OSP in the conference paper~\cite{lackner_et_al:LIPIcs.CP.2021.37} where setup times were part of the objective function. Since setup costs can be adapted in a preprocessing step in order to include any costs related to the setup times, we decided to remove setup times from the objective in favor of a simpler model formulation.} However, up to a certain extent, setup times are implicitly minimized since they can have an impact on the tardiness of jobs.

%Nonetheless there might be situations where it is impossible to schedule all jobs on time.
%Therefore, we additionally include the cumulative job tardiness in our objective function with a large weight compared to the other objectives.
%Minimizing the number of tardy jobs is considered to be a hard constraint but implemented as a soft constraint and included in the objective function with a high weight.
%This is to ensure that solutions can nonetheless be found for instances with too tight due dates, i.e. that would be unsatisfiable if due due dates were implemented as hard constraint.
Using the three-field notation introduced by Graham et al.\ \cite{graham1979optimization}, the OSP can be classified as $\Tilde{P}\vert r_j, \Bar{d_j}, maxt_j, b_i, ST_{sd,b}, SC_{sd,b}, \mathcal{E}_j, Av_m \vert \text{obj}$, where $\mathcal{E}_j$ stands for eligible machines and $Av_m$ for availability of machines. 
A more formal description of the problem constraints and the objective function obj is given in Section~\ref{sec:cp-model}.

%\subsection{Computational complexity}

As shown by Uzsoy~\cite{uzsoy1994scheduling}, minimizing makespan on a single batch processing machine is an \NP-hard problem. This can be seen as follows: For the special case that all jobs have identical processing times, minimizing makespan on a single batch processing machine is equivalent
to a bin packing problem where the bin capacity is equal to the machine capacity and item sizes are given by the job sizes.
It follows that the OSP is \NP-hard as well, since it generalizes this problem.
Indeed, minimizing makespan on a single batch processing machine is equivalent to minimizing batch processing times in an oven scheduling problem with a single machine, a single attribute, no setup times, and for which all jobs are available at the start of the scheduling horizon.

\subsection{Instance parameters of the OSP}
\label{sec:input}

%Let us now formally describe the input of the OSP. 

An instance of the OSP consists of a set $\mathcal{M}=\left\{1, \ldots, k\right\}$ of machines, a set $\mathcal{J}=\left\{1, \ldots, n\right\}$ of jobs and a set $\mathcal{A} = \left\{1, \ldots, a\right\}$ of attributes as well as the length $l \in \mathbb{N}$ of the scheduling horizon.
Every machine $m \in \mathcal{M}$ has a maximum capacity $c_m$ and an initial state $s_m \in \mathcal{A}$.
The machine availability times are specified in the form of time intervals $[as(m,i), ae(m,i)]\subseteq[0,l]$ where $as(m,i)$ denotes the start and $ae(m,i)$ the end time of the $i$-th interval on machine $m$. %This means that a batch scheduled on machine $m$ must be processed entirely within one of these intervals $[as(m,i), ae(m,i)]$; setup times must also lie within the same interval as the following batch. 
W.l.o.g. we assume that every machine has the same number of availability intervals and denote this number by $I$ where some of these intervals might be empty (i.e. $as(m,i) =  ae(m,i)$).
Moreover, availability intervals have to be sorted in increasing order (i.e. $as(m,i) \leq  ae(m,i) \leq as(m,i+1)]$ for all $i \leq I-1$).

Every job $j \in \mathcal{J}$ is specified by the following list of properties:
\begin{itemize}
        \item A set of eligible machines
          \( \mathcal{E}_j \subseteq \mathcal{M}\). %The job $j$ can only be processed on the machines in $\mathcal{E}_j$.
        \item An earliest start time (or release time) $et_j \in \mathbb{N}$ with \( 0 \leq et_j < l \). %A job cannot be processed before its earliest start time.
        \item A latest end time (or due date) $lt_j \in \mathbb{N}$ with
            \( et_j < lt_j \leq l\). %It is desired that the processing of a job is completed before its latest end time; this is however only a soft constraint.
        \item A minimal processing time $mint_j \in \mathbb{N}$ with $min_T \leq mint_j \leq max_T$, where $min_T > 0$ is the overall minimum and $max_T \leq l$ is the overall maximum processing time. %This is the minimum amount of time for which a job must be processed in one of the machines. Note that this time is independent of the machine.
         \item A maximal processing time $maxt_j \in \mathbb{N}$ with $mint_j \leq maxt_j \leq max_T$. %This is the maximum amount of time a job can spend in one of the machines. The processing time of a job must thus lie in the interval $[mint_j, maxt_j]$.
        \item A size $s_j \in \mathbb{N}$.
        \item An attribute $a_j \in \mathcal{A}$.
\end{itemize}

Moreover, an $(a \times a)$-matrix of setup times $st=(st(a_i, a_j))_{1 \leq a_i, a_j \leq a}$ and an $(a \times a)$-matrix of setup costs $sc=(sc(a_i, a_j))_{1 \leq a_i, a_j \leq a}$ are given to denote the setup times (resp. costs) incurred between a batch with attribute $a_i$ and a subsequent batch with attribute $a_j$. 
%of jobs of attribute $a_i$ and a batch of jobs of attribute $a_j$ are $st(a_i, a_j)$ (resp. $sc(a_i, a_j)$).
Setup times (resp. costs) are integers in the range $[0,max_{ST}]$ (resp. $[0,max_{SC}]$), where $max_{ST} \leq l$ (resp. $max_{SC} \in \mathbb{N}$)  denotes the maximal setup time (resp. maximal setup cost). Note that these matrices are not necessarily symmetric. 

\subsection{Example instance with six jobs}
\label{sec:example}

Consider the following example for an OSP instance consisting of six jobs, two attributes, two machines and a scheduling horizon of length $l=15$. The instance parameters are summarized in the following tables and matrices:

\begin{multicols}{2}
\begin{center}
{
\begin{tabular}{p{2.8cm}|p{1.2cm}p{1.2cm}}
Machine & $M_1$ & $M_2$   \\\hline
$c_m$ & 100 & 150 \\\hline
$s_m$ & 1 & 2 \\\hline
Availability \mbox{intervals} $[as, ae]$ & [0,6] [8,14] & [2,10] [11, 14]
\end{tabular}
}
\end{center}
\columnbreak
\begin{align*}
st & = \begin{pmatrix}
1 & 2 \\
3 & 1
\end{pmatrix}  \\
& \\
sc & = \begin{pmatrix}
0 & 20  \\
10 & 0
\end{pmatrix}
\end{align*}
\end{multicols}

\begin{center}
%\small{
\begin{tabular}{l|cccccc}
Job $j$ & 1 & 2 & 3 & 4 & 5 & 6  \\\hline
$\mathcal{E}_j$ & $M_1$ & $M_1$ & $M_1$ & $M_1$& &\\
& & $M_2$ & & $M_2$ & $M_2$& $M_2$ \\\hline
$et_j$ & 2 & 0& 0& 3& 0& 2\\
$lt_j$  &  10 & 10 & 20 & 20 & 20 & \\
$mint_j$   &  3& 3& 3& 5& 5& 5\\
$maxt_j$    &  3& 5& 5& 8& 8& 10\\
$s_j$     &  40& 60& 30& 50& 50& 50\\
$a_j$      &  2& 2 &1 & 1& 1& 1\\
\end{tabular}
%}
\end{center}

An optimal\footnote{This solution is optimal with respect to weights $w_t =4$, $w_{sc} = 1$ and $w_t= 100$ as defined in Section~\ref{sec:cp-model:obj}.} solution of this instance consists of three batches and is visualised in Figure~\ref{fig:example_solution}. In this visualisation, the dark grey areas correspond to time intervals for which the machine is not available, light gray rectangles before batches are setup times and the dashed lines represent the machine capacities.

\begin{figure}
    \centering
    \includegraphics[trim=0 5cm 0 0cm, clip, scale=0.45]{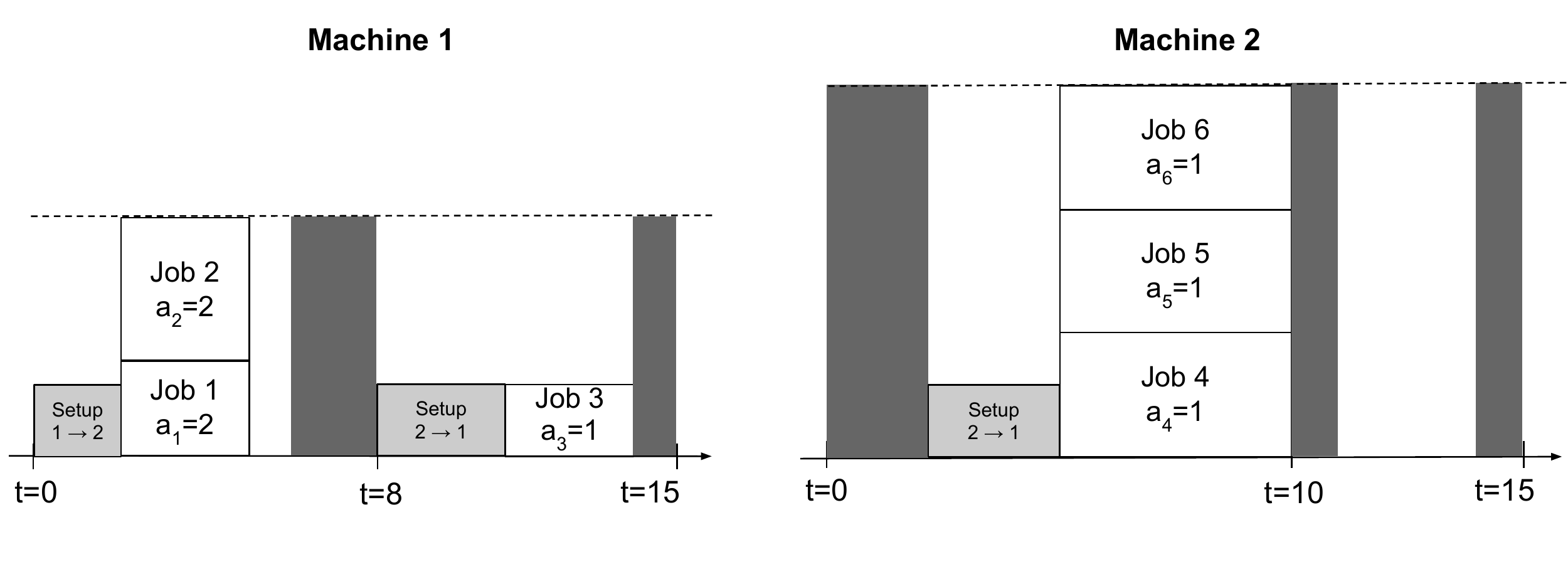}
    \caption{An optimal solution to the OSP for a small example problem.}%consisting of six jobs, two attributes and two machines.}
    \label{fig:example_solution}
\end{figure}

\section{Formal problem definition and CP-model for the OSP}
\label{sec:cp-model}

In this section we provide a formal definition of the OSP %Oven Scheduling Problem 
that will also serve as CP-model.
We first explain how batches are modeled and afterwards define decision variables, objective function, and the set of constraints.
%; we use the input parameters from Section~\ref{sec:input}. 
%A priori we do not know how many batches will be needed in total and per machine;

In the worst case we need as many batches as there are jobs; we thus define the set of potential batches as $\mathcal{B} = \{ B_{1,1}, \ldots, B_{1,n}, \ldots, B_{k,1}, \ldots, B_{k,n} \}$ to model up to $n$ batches for machines $1$ to $k$.
%The batch $B_{m,b}$ is the $b$-th batch on machine $m$.
In order to break symmetries in the model, we further enforce that batches are sorted in ascending order of their start times and empty batches are scheduled at the end.
That is, $B_{m,b+1}$ is the batch following immediately after batch $B_{m,b}$ on machine $m$ for $b \leq n-1$.
Clearly, at most $n$ of the $k \cdot n$ potential batches will actually be used and the rest will remain empty.
%How empty batches are treated is explained later on at the very end of Section~\ref{sec:cp-model:constraints}.

\subsection{Variables}
We define the following decision variables: 
\begin{itemize}
\item Machine assigned to job:    $ M_{j} \in \mathcal{M} \quad \forall j \in \mathcal{J}$
\item Batch number assigned to job:\footnote{Throughout the paper, we write $[n]$ for the interval of integers $\left\{1, \ldots, n\right\}$.}   $ B_{j} \in [n] \quad \forall j \in \mathcal{J} $
\item Start times of batches:
    $ S_{m,b} \in [0, l] \subset \mathbb{N} \quad \forall m \in \mathcal{M} \quad \forall  b \in [n] $
\item Processing times of batches: 
    $ P_{m,b} \in [0, max_T] \subset \mathbb{N} \quad \forall m \in \mathcal{M} \quad \forall  b \in [n]$
    %Since the objective is to minimize oven runtime, a batch will never run longer than necessary and thus $max_T$ is an upper bound for the processing time of every batch.
\end{itemize}

Note that $M_{j}$ and $B_{j}$ determine to which batch job $j$ is assigned ($B_{(M_j),(B_j)}$).

%Besides the aforementioned decision variables, we introduce the following helper variable arrays for the formal definition of the objective function and constraints: 
\noindent
We additionally define the following auxiliary variables:
\begin{itemize}
\item  Attribute of batch:
    $A_{m,b} \in [a] \quad  \forall m \in \mathcal{M} \quad \forall  b \in [n] $
\item Availability interval for batch: 
    $I_{m,b}  \in [I] \quad  \forall m \in \mathcal{M} \quad \forall  b \in [n]$
\item Number of batches per machine: 
    $    b_m \in [n] \quad  \forall m \in \mathcal{M}$
\end{itemize}

\subsection{Objective function}
\label{sec:cp-model:obj}

The objective function consists of three\footnote{The objective function of original formulation of the OSP in the conference paper~\cite{lackner_et_al:LIPIcs.CP.2021.37} has four components and also includes setup times. However, setup costs can be adapted in a preprocessing step in order to include any costs related to the setup times. We have thus decided to remove setup times from the objective in favor of a simpler model formulation.} components: the cumulative batch processing time across all machines $p$, the number of tardy jobs $t$ and the cumulative setup costs $sc$:

\begin{equation}\label{eqn:obj-components}
\begin{aligned}  
p &= \sum_{m \in \mathcal{M}, 1 \leq b \leq n} P_{m,b} \\
 sc & = \sum_{m \in \mathcal{M}}sc(s_m,A_{m,1}) + \sum_{\substack{m \in \mathcal{M}\\ 1 \leq b \leq b_m-1}} sc(A_{m,b},A_{m,b+1})  \\
t & = \left \vert \left\{ j \in \mathcal{J}
: S_{M_j,B_j} + P_{M_j,B_j} > lt_j\right\} \right \vert 
%  st & = \sum_{\substack{m \in \mathcal{M}\\ 1 \leq b \leq b_m-1}} st(A_{m,b},A_{m,b+1})
\end{aligned}
\end{equation}
Note that the cumulative batch processing time is the sum of  processing times of all batches, not of all jobs. Minimizing the cumulative batch processing time thus leads to batches being formed in the most efficient way.
Setup costs are composed of the setup costs before the first batch on every machine, which are determined by the initial state of the machine and the attribute of the first batch, and the setup costs between consecutive batches, which are determined by the attributes of the batches.

\noindent \emph{Normalization of cost components.}
Finding a trade-off among several objectives is the topic of multiobjective optimization \cite{deb2014multi,miettinen2012nonlinear}. In practice, lexicographic optimization is often chosen.
For our specific application, we wanted an approach that allows high flexibility and lets users decide on the importance of each objective. 
%In practice it is important to define an objective function that is highly flexible and configurable.
We thus decided in consultation with our industrial partner to define the objective function as a linear combination of the three components.
Therefore, the components $p$, $sc$ and $t$ need to be normalized to $\tilde{p}$,  $\tilde{sc}$ and $\tilde{t} \in [0,1]$:
%, which we achieve with instance-dependent upper bounds:
\begin{equation}\label{eqn:obj-components-normalized}
\begin{aligned}  
\tilde{p} & = \frac{p}{ avg_t \cdot n} \text{ with } avg_t
 = \left\lceil \frac{\sum_{j \in \mathcal{J}} mint_j}{n} \right \rceil
\\
% \tilde{st} & = \frac{st}{\max(max_{ST},1) \cdot n}
  \tilde{sc} & = \frac{sc}{\max(max_{SC},1) \cdot n} \\
  \tilde{t} & = \frac{t}{n} 
\end{aligned}
\end{equation}

In the worst case, every batch processes a single job.
%every job needs to be processed in a batch of its own and the cumulative batch processing time is thus equal to the sum of all minimal job processing times, or
In this case the total batch processing time is $n$ times the average job processing time $avg_t$.
The setup  costs are bounded by the maximum setup  cost multiplied with the number of jobs.
We take the maximum of 1  $max_{SC}$ since it is possible that $max_{SC}=0$.
The number of tardy jobs is clearly bounded by the total number of jobs.

Finally, the objective function $\text{obj}$ is a linear combination of the three normalized components:
\begin{align}\label{eqn:obj-function}
\text{obj}=(w_p\cdot \tilde{p} +  w_{sc} \cdot \tilde{sc} + w_t \cdot \tilde{t})/(w_p + w_{sc} + w_t) \quad \in [0,1] \subset \mathbb{R}
\end{align}
where the weights $w_p$, $w_{sc}$ and $w_t$ take integer values. 
Together with our industrial partner, we chose the default values to be $w_p=4$, $w_{sc}=1$ and $w_t=100$, which captures the requirements of typical practical scheduling applications.
The high weight given to the minimization of the number of tardy jobs reflects that t is a soft constraint: Jobs are allowed to finish late, this is however highly undesirable. Among the other two components, minimizing the cumulative batch processing time has a greater priority than minimizing setup costs.

\noindent\emph{Integer-valued objective.}
As some state-of-the-art CP solvers can only handle integer domains, we propose an alternative objective function $\text{obj}'$, where we additionally multiply $\tilde{p}$, $\tilde{sc}$, and $\tilde{t}$ by the number of jobs and the least common multiple of $avg_t$ and $max_{SC}$:
\begin{align}\label{eqn:obj-function-int}
\text{obj}' & = C \cdot n \cdot (w_p + w_{sc} + w_t) \cdot \text{obj} \quad \in %[0,C \cdot n \cdot(\alpha + \beta + \gamma +\delta)] \subseteq 
\mathbb{N} \notag \\
& =  \frac{w_p\cdot C}{avg_t}\cdot p  
+ \frac{w_{sc} \cdot C}{\max(max_{SC},1)} \cdot sc   
+ w_t \cdot C \cdot t, \\
\text{where } C  & = \texttt{lcm}(avg_t, \max(max_{SC},1)). \notag
\end{align}
Preliminary experiments using the MIP solver Gurobi showed that using $\text{obj}$ and $\text{obj}'$ both lead to similar results.
We therefore used only $\text{obj}'$ in our final experimental evaluation. % in this paper.
%lead to any improvements in the runtime of the model or in the solution quality. We thus use the integer-valued objective in equation~\eqref{eqn:obj-function-int} throughout all our models.

%Let us also note that for the calculation of $t$ in our CP model we implicitly utilize the element constraint to use variables as indices of arrays.

\subsection{Constraints}
\label{sec:cp-model:constraints}

In what follows, we formally define the constraints of the OSP using a high-level CP modeling notation. 
Most of these constraints can directly be handled by CP solvers, however we implicitly make use of constraint reification to express conditional sums and additionally use the maximum global constraint.
Furthermore, we implicitly utilize the element constraint to use variables as array indices.

\begin{itemize}
    \item Jobs may not start before their earliest start time:
    $
        S_{M_j, B_j} \geq et_{j} \quad \forall j \in \mathcal{J}.
    $
    \item Batch processing times must lie between the minimal and maximal processing time of all assigned jobs:
    % \[
    %     mint_j \leq P_{m,b} \leq maxt_j \quad \forall j \in \mathcal{J} \text{ for which } B_j=b \text{ and } M_j=m.
    % \]
    %The batch processing time is thus given as the maximum over all minimal processing times the assigned jobs:
    \begin{gather*}
    P_{m,b} = \max(mint_j: j \in \mathcal{J} \text{ with } B_j=b \land M_j=m) \quad 
    \forall m \in \mathcal{M}, b \in [b_m] \\
    P_{M_j,B_j} \leq maxt_j \quad \forall j \in \mathcal{J}
    \end{gather*}
    \item Batches on the same machine may not overlap and setup times must be considered between consecutive batches: 
    \begin{gather*}
        S_{m,b} + P_{m,b}  + st_{(A_{m,b},A_{m,b+1})} \leq S_{m,b+1} \quad \forall b \in [b_m -1],
    \end{gather*}
    % using the symmetry-breaking constraint that batches are sorted in ascending order of their start times. 
    \item Batches and the preceding setup times must lie entirely within one machine availability interval. 
    % \begin{flalign*}
    %     \exists i \in [I]: \quad & S_{m,b} + P_{m,b} & \leq & ae(m,i) \quad \forall m \in \mathcal{M}, \forall b \in [b_m] \\
    %     & S_{m,b} - st_{(A_{m,b-1},A_{m,b})}  & \geq &  as(m,i) \quad \forall m \in \mathcal{M}, \forall b \in \{2, \ldots, b_m \}.
    % \end{flalign*}
    This is modeled with auxiliary variables $I_{m,b}$ which encode in which availability interval batch $B_{m,b}$ lies. The last set of constraints concerns the first batch on every machine:
    \begin{flalign*}
     &I_{m,b}  = \max(i \in [I]: S_{m,b} \geq as(m,i)) \quad &\forall m \in \mathcal{M} \quad \forall b \in [b]\\
     &S_{m,b} + P_{m,b}  \leq  ae(m,I_{m,b}) \quad &\forall m \in \mathcal{M} \quad \forall b \in [b_m]\\
     &as(m,I_{m,b}) \leq S_{m,b} - st(A_{m,b-1},A_{m,b})    \quad &\forall m \in \mathcal{M}, \forall b \in \{2, \ldots, b_m \} \\
     &as(m,I_{m,1}) \leq S_{m,1} - st(s_m,A_{m,1})   \quad &\forall m \in \mathcal{M}.
    \end{flalign*}
    In practice, an interval for which a machine is unavailable can correspond to a period during which a machine is turned off or occupied with some other task, or to a period for which the personnel required for the setup and running of ovens is unavailable. For all described cases, the setup of machines for a given batch cannot be performed during an interval for which the machine is unavailable and must precede the batch immediately.
    Therefore, setup times also need to fall completely within the associated availability interval.
    \item Total batch size must be less than machine capacity:
    \begin{gather*}
      \sum_{j \in \mathcal{J}:M_j=m \wedge B_j=b}s_j  \leq c_m \quad \forall m
      %{\substack{j \in \mathcal{J},\\M_j=m }}
      \in M \text{ with } b_m > 0, \forall b \in [b_m]
    \end{gather*}
    \item Jobs in one batch must have the same attribute,
    % \[
    % B_{j}=B_{i} \wedge M_{j} = M_{i} \implies a_j = a_i \forall i, j \in \mathcal{J},
    % \]
    which we model with auxiliary variables $A_{m,b}$ to set the attribute of a batch:
    \[
    A_{M_j,B_j}=a_j \quad \forall j \in \mathcal{J}.
    \]
    \item The assigned machine must be eligible for a job:
    $
     M_j \in \mathcal{E}_j \quad \forall j \in \mathcal{J}.
    $
%\end{compactitem}
%It remains to specify the constraints for the number of batches per machine and for empty batches:
%\begin{itemize}
\item Set the number of batches per machine variables:
    $
    b_{M_j} \geq B_j \quad \forall j \in \mathcal{J}
    $
%    The number $b_m$ of batches on machine $m$ is equal to the largest $b \in [1,n]$ for which $B_j=b$ holds for some job $j \in \mathcal{J}$.
\item Set variables for empty batches (i.e. batches $B_{m,b}$ with $b > b_m$):
\begin{equation*}
    \begin{aligned}
        &S_{m,b} =l \quad
        &P_{m,b}  =0 \quad 
        &A_{m,b}  = 1 \quad
        &I_{m,b} = I &\quad \forall m \in \mathcal{M}, b_m < b \leq n
    \end{aligned}
    \end{equation*}
\end{itemize}

%Finally, let us remark that the constraint that all jobs should have finished processing before their latest end time is not modeled as a hard constraint. As elaborated earlier for the definition of the objective function in Section~\ref{sec:cp-model:obj}, it is instead implemented as a soft constraint incurring a high cost.

\section{ILP-model for the OSP}
\label{sec:mip-model}
%In the following we present our MIP model; in contrast to the CP model presented in the previous section, all constraints are linear here. Some of the constraints can be formulated in the same way as for the CP model, however we present the MIP model in its entirety so that it can be understood independently of Section~\ref{sec:cp-model}. 
%In this section, we propose a ILP-formulation of the problem.
We propose an ILP-formulation for the OSP, where batches are modeled similarly as in the CP-model: for $m \in \mathcal{M}, b \in [n]$, $B_{m,b}$ models the batch at position $b$ on machine $m$.
However, we use a different set of decision variables: Binary variables $X_{m,b,j}$ encode whether job $j$ is assigned to batch $B_{m,b}$ ($X_{m,b,j}=1 \Leftrightarrow (B_j = b \land M_j = m)$), and integer variables $S_{m,b}$ and $P_{m,b}$ encode the start and processing times of batches. 
Moreover, we use auxiliary variables  $st_{m,b}$ and $sc_{m,b}$ for setup times and costs between batch $b$ on machine $m$ and the following batch.

To handle empty batches, we define an additional attribute with value 0 and extend the matrices of setup times $\bar{st}$ and setup costs $\bar{sc}$ so that no costs occur when transitioning from an arbitrary batch to an empty batch: $\bar{st}(a_i,a_j)=\bar{sc}(a_i,a_j)=0$ if $a_i=0$ or $a_j$=0.
%Setup times and setup costs from an arbitrary batch to an empty batch are thus $0$ as one would expect.
Moreover, we add a machine availability interval $[l,l]$ of length 0 to the list of availability intervals so that empty batches can be scheduled for this interval (the maximum number of intervals per machine therefore becomes $I+1$).
We then model the problem as follows:\footnote{If not stated otherwise, $\forall m$ is short for $\forall m \in \mathcal{M}$, $\forall b$ for $\forall b \in \left[1,n\right]$, $\forall i$  for $\forall i \in [1,I+1]$ and $\forall j$ for $\forall j \in \mathcal{J}$.}

\begin{align}
\text{Min.} \quad & \text{obj'}=\tilde{w_p}\cdot p +  \tilde{w_{sc}} \cdot sc + \tilde{w_t} \cdot t, \text{ where} &\label{eqn:mip-obj}\\ 
 &p = \sum_{\substack{m \in \mathcal{M}\\ 1 \leq b \leq n}} P_{m,b}, %& \notag\\
\qquad \qquad  t = \sum_{\substack{j \in \mathcal{J}, m \in \mathcal{M}\\ 1 \leq b \leq n}} T_{m,b,j},& \notag \\ 
&  sc = \sum_{\substack{m \in \mathcal{M}\\ 1 \leq b \leq n-1}} \bar{sc}_{m,b}, \notag \\ 
\text{s.t.} \, & \sum\nolimits_{m \in \mathcal{M}, 1 \leq b \leq n} %\sum_{\substack{m \in \mathcal{M} \\ 1 \leq b \leq n}} 
X_{m,b,j}=1 & \forall j \label{eqn:mip_exactly_one_batch}\\
& \sum\nolimits_{m \in \mathcal{E}_j, 1 \leq b \leq n} X_{m,b,j}=1 & \forall j  \label{eqn:mip_eligible_machine}\\
& S_{m,b} \geq et_j \cdot X_{m,b,j} & \forall m, \forall b, \forall j  \label{eqn:mip_start_not_too_early} \\
& mint_j \cdot X_{m,b,j} \leq P_{m,b} \quad \wedge & \notag \\
& P_{m,b} \leq maxt_j \cdot X_{m,b,j} + max_T\cdot(1-X_{m,b,j})  & \forall m, \forall b, \forall j \label{eqn:mip_processing time}\\
& S_{m,b+1} \geq S_{m,b} + P_{m,b} + st_{m,b} & \forall m, \forall b \leq n-1 \label{eqn:mip_sorted_starting_times} \\
& \sum\nolimits_{j \in \mathcal{J}} s_j \cdot X_{m,b,j} \leq c_m & \forall m, \forall b \label{eqn:mip_macx_cap} \\
%helper variables
& a_j \cdot X_{m,b,j} \leq A_{m,b} \quad \wedge & \notag \\
& A_{m,b} \leq a_j \cdot X_{m,b,j} + a\cdot(1-X_{m,b,j})  & \forall m, \forall b, \forall j \label{eqn:mip_matching_attributes}\\
%availability intervals
& as(m,i) \cdot I_{m,b,i} \leq S_{m,b} \quad \wedge & \notag \\
&S_{m,b} \leq ae(m,i) \cdot I_{m,b,i} + l \cdot (1-I_{m,b,i})& \forall m, \forall b, \forall i \label{eqn:mip_avail_interval}\\
&\sum\nolimits_{1\leq i \leq I+1}I_{m,b,i} = 1 & \forall m, \forall b \label{eqn:mip_exactly_one_avail_interval}\\
& as(m,i) \cdot I_{m,b,i} \leq S_{m,b} - st_{m,b-1} & \forall m, \forall b \geq 2, \forall i & \label{eqn:mip_batch_fits_avail_interval1}\\
& as(m,i) \cdot I_{m,1,i} \leq S_{m,1} - st(s_m, A_{m,1}) & \forall m,  \forall i & \label{eqn:mip_first_batch_fits_avail_interval1}\\
\textcolor{white}{\text{s.t.}}&S_{m,b} + P_{m,b} \leq ae(m,i) \cdot I_{m,b,i} + l \cdot (1-I_{m,b,i}) &\forall m, \forall b, \forall i\label{eqn:mip_batch_fits_avail_interval2}\\
%tardy jobs
& T_{m,b,j} \leq X_{m,b,j} & \forall j, \forall m, \forall b \label{eqn:mip_latejob1}\\
& S_{m,b} + P_{m,b}\leq (X_{m,b,j}-T_{m,b,j})\cdot(lt_j-l) + l & \forall j, \forall m, \forall b \label{eqn:mip_latejob2}\\
& S_{m,b} + P_{m,b} + (1-T_{m,b,j})\cdot(l+1) > lt_j 
& \forall j, \forall m, \forall b \label{eqn:mip_latejob3}\\
%empty batches
& \sum\nolimits_{j \in \mathcal{J}} X_{m,b,j} \geq  1 - E_{m,b} & \forall m, \forall b \label{eqn:mip_empty_batch1}\\
& X_{m,b,j} \leq 1 - E_{m,b} & \forall m, \forall b, \forall j\label{eqn:mip_empty_batch2} \\
& S_{m,b} \geq l \cdot E_{m,b} & \forall m, \forall b \label{eqn:mip_empty_batch_start} \\
& P_{m,b} <= max_T \cdot (1-E_{m,b}) & \forall m, \forall b \label{eqn:mip_empty_batch_duration} \\
& E_{m,b} <= I_{m,b,I+1} & \forall m, \forall b \label{eqn:mip_empty_batch_interval} \\
& A_{m,b} <= a \cdot (1-E_{m,b}) & \forall m, \forall b \label{eqn:mip_empty_batch_attribute} \\
& E_{m,b} <= E_{m,b+1}& \forall m, \forall b \leq n-1 \label{eqn:mip_empty_batch_grouped} \\
& st_{m,b} = \bar{st}(A_{m,b},A_{m,b+1})  & \forall m, \forall b \leq n-1 \label{eqn:mip_setuptime_between_batches} \\
& sc_{m,b} = \bar{sc}(A_{m,b},A_{m,b+1}) & \forall m, \forall b \leq n-1 \label{eqn:mip_setupcost_between_batches}\\
& X_{m,b,j} \in \{0,1\}, S_{m,b} \in [0,l], P_{m,b} \in [0,max_T],  & \notag \\
& A_{m,b} \in \left[0,a\right], I_{m,b,i} \in \{0,1\}, E_{m,b} \in \{0,1\}, & \notag \\
& st_{m,b} \in [0,max_{ST}], sc_{m,b} \in [0,max_{SC}]& \forall m, \forall b, \forall j \label{eqn:mip_domains}
\end{align}

The weights of the objective function~\eqref{eqn:mip-obj} are as described in equation~\eqref{eqn:obj-function-int} in Section~\ref{sec:cp-model:obj}. 
%The decision variables $X_{m,b,j}$ are binary variables where $X_{m,b,j}=1$ if and only if job $j$ is assigned to batch $B_{m,b}$.
Constraint~\eqref{eqn:mip_exactly_one_batch} ensures that every job is assigned to exactly one batch. Moreover,  constraint~\eqref{eqn:mip_eligible_machine} ensures that jobs can only be assigned to eligible machines. %as for every job $j$ it must hold that $X_{m,b,j}=1$ for some $m \in \mathcal{E}_j$.
%The decision variables $S_{m,b}$ and $P_{m,b}$ are the start time respectively processing time of batch $B_{m,b}$.
Constraint~\eqref{eqn:mip_start_not_too_early} specifies that a batch may not start before the earliest start of any job in the batch.
The processing time of a batch is constrained by equations~\eqref{eqn:mip_processing time}.
%: It has to be greater or equal to the minimal processing time $mint_j$ and less than or equal the maximal processing time $maxt_j$ of every assigned job $j$.
%If $X_{m,b,j}=0$, i.e., job $j$ does not lie in batch $B_{m,b}$, the inequalities simply become $0 \leq P_{m,b} \leq max_T$.
Constraint~\eqref{eqn:mip_sorted_starting_times} imposes additional restrictions on the starting times and ensures that the correct setup times are considered between consecutive batches.
%To break symmetries, the batches on a machine need to be sorted in ascending order of their start times.
%Moreover, the start time $S_{m,b+1}$ of the $(b+1)$-st batch on machine $m$ has to be larger or equal to the sum of the start time $S_{m,b}$ and the processing time $P_{m,b}$ of the preceding batch $B_{m,b}$ plus the setup time $st_{m,b}$ between the two batches. The setup times between batches are defined later on in Equation~\eqref{eqn:mip_setuptime_between_batches}.
Constraint~\eqref{eqn:mip_macx_cap} ensures that the machine capacities are not exceeded for any batch.%; for a given machine $m$ and a batch $b$ the sum adds up exactly the sizes $s_j$ of jobs $j$ that are scheduled in $B_{m,b}$.
%The variables in equations~\eqref{eqn:mip_matching_attributes} to \eqref{eqn:mip_setupcost_between_batches} are auxiliary variables.
Constraint \eqref{eqn:mip_matching_attributes} ensures that jobs in the same batch have the same attribute. 
%by linking the job's attribute to the batch's attribute (Auxiliary variables $A_{m,b} \in \left[0,a\right]$ set the attribute associated to batches $B_{m,b}$.). Therefore, if job $j$ is in batch $B_{m,b}$, i.e. $X_{m,b,j}=1$, it has to hold that $A_{m,b}=a_j$.
%; otherwise the inequalities simply imply that $0 \leq A_{m,b} \leq a$.
The binary auxiliary variables $I_{m,b,i}$ in constraint~\eqref{eqn:mip_avail_interval} encode whether batch $B_{m,b}$ is scheduled within the $i$-th availability interval $\left[as(m,i), ae(m,i)\right]$ of machine $m$. Therefore, if $I_{m,b,i}=1$, it must hold that $as(m,i) \leq S_{m,b} \leq ae(m,i)$.
%; otherwise the inequality simply is $0 \leq S_{m,b} \leq 1$.
The redundant constraint~\eqref{eqn:mip_exactly_one_avail_interval} ensures that every batch is scheduled within exactly one availability interval.
Constraints~\eqref{eqn:mip_batch_fits_avail_interval1}-- \eqref{eqn:mip_batch_fits_avail_interval2} ensure that the entire processing time of batch $B_{m,b}$ as well as the preceding setup times $st_{m,b-1}$ (%as defined in
see~\eqref{eqn:mip_setuptime_between_batches}) lie within a single availability interval. Constraint~\eqref{eqn:mip_first_batch_fits_avail_interval1} takes care of the special case of the first batch on a machine.
%That is, if batch $B_{m,b}$ is scheduled for the $i$-th availability interval $\left[as(m,i), ae(m,i)\right]$ of machine $m$, the batch setup must start after $as(m,i)$ and the batch must have finished processing before $ae(m,i)$.

The binary auxiliary variables $T_{j,m,b}$ encode whether job $j$ in batch $B_{m,b}$ finishes after its latest end date and is used to calculate the number of tardy jobs $t$.
Constraint~\eqref{eqn:mip_latejob1} ensures that $T_{j,m,b}=1$ is only possible if job $j$ is assigned to batch $B_{m,b}$. If $T_{j,m,b}=0$, job $j$ must finish before $lt_j$ (Constraint~\eqref{eqn:mip_latejob2}) and if $T_{j,m,b}=1$, it must hold that $S_{m,b} + P_{m,b} > lt_j$ (Constraint~\eqref{eqn:mip_latejob3}).
The binary variables $E_{m,b}$ in equations~\eqref{eqn:mip_empty_batch1} to \eqref{eqn:mip_empty_batch_grouped} encode whether batch $B_{m,b}$ is empty or not. Constraints~\eqref{eqn:mip_empty_batch1} and \eqref{eqn:mip_empty_batch2} ensure that $E_{m,b}=1$ iff no job is scheduled for batch $B_{m,b}$.
The constraints~\eqref{eqn:mip_empty_batch_start} to \eqref{eqn:mip_empty_batch_interval} set the start times, processing times, availability intervals, and attributes for empty batches.
%Empty batches are scheduled at the end of the scheduling horizon and have processing time 0. Accordingly, empty batches are scheduled for the $I+1$-th availability interval $[l,l]$. The attribute of empty batches is set to the additional attribute 0. 
Moreover, in order to break symmetries, the list of batches $(B_{m,b})_{1 \leq b \leq n}$ per machine $m \in \mathcal{M}$ is sorted so that all non-empty batches appear first (constraint~\eqref{eqn:mip_empty_batch_grouped}).
%That is, if there are a total of $k$ non-empty batches for some machine $m$, $E_{m,b}=0$ for all $b \leq k$ and $E_{m,b}=1$ for all $b > k$.
Constraint~\eqref{eqn:mip_setuptime_between_batches} defines the setup times $st_{m,b}$ and \eqref{eqn:mip_setupcost_between_batches} the setup costs $sc_{m,b}$ between consecutive batches on the same machine.
%; $st_{m,b}$ (resp. $sc_{m,b}$) denotes the setup time (resp. setup cost) between batch $B_{m,b}$ and batch $B_{m,b+1}$.
%Note that setup times and costs are also defined between empty batches and set to $0$ in this case.
%as the setup costs and times from attribute $0$ to attribute $0$ are equal to $0$.
Finally, equation~\eqref{eqn:mip_domains} defines the domains of all decision and helper variables. %Recall that $max_{ST}$ resp.\ $max_{SC}$ denote the maximal setup time resp.\ cost.

\section{CP model using representative jobs for batches}
\label{sec:model-repr-job}

For the models presented in Sections~\ref{sec:cp-model} and \ref{sec:mip-model}, we decided to create decision variables for start times and processing times of every batch that could possibly be present ($n \cdot k$ batches in total). 
Which batch a job is assigned to was then encoded using additional decision variables.
The scheduling constraints were formulated for the batches.

An alternative approach is to schedule jobs instead of batches.
However, only one \textit{representative job} for every batch is scheduled; all other jobs in the same batch point to their representative job.
In order to break symmetries, the job with lowest index is chosen to be the representative job for its batch.
Whereas the previous approach models the OSP primarily as an assignment problem (jobs are assigned to batches that have a predefined order),
this approach focuses more on the scheduling aspect of the OSP: a subset of representative jobs are chosen for which an optimal schedule needs to be found.

Since CP Optimizer is particularly well suited for scheduling problems as demonstrated by Laborie et al.~\cite{laborie2018ibm}, we formulated this modelling approach based on representative jobs as a CP model for IBM ILOG CPLEX Studio. 
This model is written using the Optimization Programming Language (OPL) \cite{hentenryck2002constraint} and makes use of \emph{interval variables}, as well as other concepts that are specific to CP Optimizer; for an introduction to the used CP Optimizer concepts, see~\cite{IBMsched2017}.
In addition, we modelled this approach using the MiniZinc modelling language in order to be able to evaluate the performance of different solvers. This alternative solver-independent CP formulation is briefly described in Section~\ref{sec:cp-repr}.

We define the following decision variables for the OSP: 
\begin{itemize}

\item Optional interval variables for jobs:    
\[
\text{interval } \text{job}_j  \text{ optional}  \subseteq [et_j,l]  
     \text{ size} \in [mint_j, maxt_j]   
	  \quad \forall j \in \mathcal{J}.
\]
The variable $\text{job}_j$ is only present if job $j$ is the representative job of its batch. Note that $\text{job}_j  \text{ optional}  \subseteq [et_j,l]$ ensures that every representative job starts after its earliest start time $et_j$.

\item Optional interval variables for jobs on machines:
\begin{align*}
\text{interval } \text{jobM}_{j,m}  &\text{ optional}  \subseteq [et_j,l]  
     \text{ size} \in [mint_j, maxt_j]
      \\& \text{ intensity } av_m \quad \forall j \in \mathcal{J} \,
	  \quad \forall m \in \mathcal{M}.
\end{align*}
The variable $\text{jobM}_{j,m}$ is only present if job $j$ is assigned to machine $m$ and is representative for its batch.
The intensity function $av_m$ encodes the machine availability times and is modeled using \emph{intensity step functions}; $av_m(t)=100$ if machine $m$ is available at time $t$ and $av_m(t)=0$ otherwise.

\item Pointers to representative jobs in the same batch:
\[
\text{inBatchWith}_j \in [0,n] \quad \forall j \in \mathcal{J}.
\]
If $\text{inBatchWith}_j=i$ for some job $i \in \mathcal{J}$, job $j$ is scheduled in the same batch as job $i$ and job $i$ is representative for this batch.
If $\text{inBatchWith}_j=0$, job $j$ is representative for its batch.

\item Interval sequence variables for the set of jobs on every machine
\begin{align*}
    \text{sequence } \text{mach}_m &\text{ in }
	\{\text{jobM}_{j,m}: j \in \mathcal{J}\} \\
	&\text{ types } \{a_{j}: j \in \mathcal{J}\} \quad \forall m \in \mathcal{M}.
\end{align*}
Interval sequence variables are used to establish the order of jobs, respectively batches, on machines; the value of $\text{mach}_m$ is a permutation of the interval variables $\text{jobM}_{j,m}$ (absent jobs are ignored).  They are used to formulate scheduling constraints such as \texttt{noOverlap}.
So-called types are used to model the attribute of jobs (the attribute of a batch containing job $j$ is equal to $a_j$).

\item Setup times between batches are also modelled using optional interval variables:
\begin{align*}
\text{interval } \text{setup}_{j,m}  &\text{ optional}  \subseteq [0,l]  
     \text{ size} \in [0, max_{ST}]
      \\& \text{ intensity } av_m \quad \forall j \in \mathcal{J} \,
	  \quad \forall m \in \mathcal{M},
\end{align*}
where $max_{ST}$ denotes the overall maximum setup time.	 
The variable  $\text{setup}_{j,m}$ is used to model the setup operation before $\text{jobM}_{j,m}$.
Setup times need to be modelled in addition to jobs themselves in order to ensure that setup times fall within the same machine availability interval as the following jobs.
\item A binary helper variable is used to model tardy jobs: \\ $\text{tardy}_j \in \{ 0,1\} \quad \forall j \in \mathcal{J}$.
\end{itemize}

For the formulation of the objective function as well as the constraints, we make use of the OPL function \texttt{typeOfPrev}. For a sequence interval variable $s$, an interval variable $i$, and two integers $b$ and $c$, $\texttt{typeOfPrev}(s,i,b,c)$ returns the type (=attribute) of the interval variable that is previous to $i$ in the sequence $s$. When $i$ is present and is the first interval in $s$, the function returns the constant integer value $b$. When $i$ is absent, the function returns the constant integer value $c$.
We then model the OSP  as follows:
\begin{align}
\text{Min.} \quad & \text{obj'}=\tilde{w_p}\cdot p +  \tilde{w_{sc}} \cdot sc + \tilde{w_t} \cdot t, \text{ with} &\label{eqn:opl-obj}\\ 
 &p = \sum_{\substack{m \in \mathcal{M}\\ j \in \mathcal{J}}}
	\texttt{lengthOf}(\text{jobM}_{j,m}), %& \notag \\
\qquad \qquad  t = \sum_{j \in \mathcal{J}} \text{tardy}_j,& \notag \\ 
&  sc = \sum_{\substack{m \in \mathcal{M}\\ j \in \mathcal{J}}} sc_{prev(j,m),a_j}, \notag 
\end{align}
where $prev(j,m) =\texttt{typeOfPrev}(\text{mach}_m, \text{jobM}_{j,m}, s_m, 0)$.
\begin{align}
\text{s.t.} \quad & \texttt{alternative}(\text{job}_j, \{ \text{jobM}_{j,m}: m \in \mathcal{M}\}) & \forall j \label{eqn:olp_exactly_one_machine}\\
& \texttt{alternative}(\text{job}_j, \{ \text{jobM}_{j,m}: m \in \mathcal{E}_m\}) & \forall j \label{eqn:olp_eligible_mach}\\
& (\texttt{presenceOf}(\text{job}_j) \land \text{inBatchWith}_j=0) & \notag \\
& \lor(\neg \texttt{presenceOf}(\text{job}_j) \land \text{inBatchWith}_j>0) & \forall j \label{eqn:olp_repr-or-pointer}\\
& \text{inBatchWith}_j=i %& \notag \\
\Rightarrow 
  \texttt{presenceOf}(\text{job}_i) \land i<j & \forall i,j \in \mathcal{J} \label{eqn:olp_pointer}\\
& \text{inBatchWith}_j=i & \notag \\
&
\Rightarrow
a_i = a_j & \notag \\
&
\phantom{\Rightarrow}
\land mint_j \leq \texttt{lengthOf}(\text{job}_i) \leq maxt_j &  \notag \\
&
\phantom{\Rightarrow}
\land \texttt{startOf}(\text{job}_i) \geq et_j & \notag \\
&\phantom{\Rightarrow}
\land \sum_{m \in \mathcal{E}_j} \texttt{presenceOf}(\text{jobM}_{i,m})=1
  & \forall i,j \in \mathcal{J} \label{eqn:olp_batch_conditions}\\
  &  
    \texttt{presenceOf}(\text{jobM}_{j,m})
  	\cdot (s_j + \sum_{\substack{i \in \mathcal{J} \text{ with}\\\text{inBatchWith}_i=j}} s_i) \leq c_m
    & \forall m, \forall j \label{eqn:opl_mach_cap} \\
& \texttt{noOverlap}(\text{mach}_m, st) & \forall m \label{eqn:opl_no_overlap} \\
&    \texttt{presenceOf}(\text{jobM}_{j,m}) = \texttt{presenceOf}(\text{setup}_{j,m}) & \forall m, \forall j \label{eqn:opl_setup1}\\   
&  \texttt{endAtStart}(\text{setup}_{j,m}, \text{jobM}_{j,m}) & \forall m, \forall j \label{eqn:opl_setup2}\\ 
&  \texttt{endOfPrev}(\text{mach}_m, \text{jobM}_{j,m} , 0 ) & \notag \\
& \qquad \leq  \texttt{startOf}(\text{setup}_{j,m})
& \forall m, \forall j \label{eqn:opl_setup3}\\
& \texttt{lengthOf}(\text{setup}_{j,m}) = st_{a_i,a_j}, & \notag \\
& \qquad \text{where } a_i  = \texttt{typeOfPrev}(\text{mach}_m, \text{jobM}_{j,m}, s_m, 0) &\forall m, \forall j \label{eqn:opl_setup_length}\\
& \texttt{forbidExtent}(\text{jobM}_{j,m}, av_m) &\forall m, \forall j \label{eqn:opl_jobs_mach_av}\\
&	\texttt{forbidExtent}(\text{setup}_{j,m}, av_m) &\forall m, \forall j \label{eqn:opl_setup_mach_av}\\
&(\texttt{presenceOf}(\text{job}_{j}) \land \texttt{endOf}(\text{job}_{j}) > lt_j) & \notag \\
 & \lor
  (\text{inBatchWith}_j=i \land \texttt{endOf}(\text{job}_{i}) > lt_j)
 & \notag \\
 &\Rightarrow
  \text{tardy}_j=1 & \forall i,j \in \mathcal{J} \label{eqn:olp_tardy_jobs}\\
& \sum_{m \in \mathcal{M}, j \in \mathcal{J}}
	\texttt{presenceOf}(\text{jobM}_{j,m}) \leq n& \label{eqn:olp_batch_count}
\end{align}

The \texttt{alternative} constraint in equation~\eqref{eqn:olp_exactly_one_machine} ensures that every job $j$, whenever present, is scheduled on exactly one machine and that the interval variables $\text{job}_j$ and $\text{jobM}_{j,m}$ are synchronised, where $m$ is the machine $j$ is assigned to.
Constraint~\eqref{eqn:olp_eligible_mach} ensures that every job $j$, whenever present, is scheduled on an eligible machine.
Constraints~\eqref{eqn:olp_repr-or-pointer} and \eqref{eqn:olp_pointer} ensure that every job is either scheduled or points at some other job that is scheduled and has a lower index.

Constraint~\eqref{eqn:olp_batch_conditions} concerns the conditions that need to be fulfilled for all jobs in a batch: every job needs to have the same attribute as the representative job for the batch, the length of the interval variable corresponding to the representative job must be greater than or equal to the minimum processing time and smaller than or equal to than the maximum processing time of every job in the batch and the start of the representative job may not before the earliest start time of any other job in the batch.
Note that the condition on the processing times and the earliest start time for the representative job are fulfilled by definition of the job interval variables.
Moreover, it is ensured that jobs can only be processed together with another job if the assigned machine is eligible.
Constraint~\eqref{eqn:opl_mach_cap} guarantees that machine capacities are not exceeded.

The \texttt{noOverlap}-constraint in equation~\eqref{eqn:opl_no_overlap} ensures that representative jobs, i.e. batches, on the same machine do not overlap. Passing the setup time matrix $st$ as additional parameter to the \texttt{noOverlap}-constraint enforces a minimal distance between consecutive jobs based on the types (=attributes) of the jobs. This constraint is sufficient to ensure that enough time is left between batches to allow for the setup operation to be performed.
However, we also need to guarantee that the setup operation falls within the same machine availability interval as the following batch.

Therefore, setup times are also modelled as interval variables and have to fulfill constraints~\eqref{eqn:opl_setup1}--\eqref{eqn:opl_setup_length}:
the setup time interval before job $j$ on machine $m$  is present if and only if  job $j$ is present on machine $m$ (constraint~\eqref{eqn:opl_setup1}), setup times end exactly at the start of the following job (constraint~\eqref{eqn:opl_setup2}) and start after the preceding job on the same machine, or at a time $\geq 0$ if the following job is the first on this machine (constraint~\eqref{eqn:opl_setup3}).
The length of the setup time interval is given by the entry $st(a_i, a_j)$ in the setup time-matrix, where $a_i$ is the attribute of the preceding and $a_j$ of the following job on the machine.
Constraints~\eqref{eqn:opl_jobs_mach_av} and \eqref{eqn:opl_setup_mach_av} finally force jobs and setup times to fall entirely within a machine availability interval of the assigned machine:  whenever $\text{jobM}_{j,m}$ or $\text{setup}_{j,m}$ is present, it cannot overlap a point $t$ where $av_m(t)=0$.

The helper variable encoding tardy jobs is set via constraint~\eqref{eqn:olp_tardy_jobs}.
Finally, equation~\eqref{eqn:olp_batch_count} is a redundant constraint that states that the total number of batches is smaller or equal to the total number of jobs.

\section{Alternative model implementations and search strategies}
\label{sec:alternative-implementations+search}

To make a thorough comparison of our CP models possible, we also implemented the CP model presented in Section~\ref{sec:cp-model} using OPL and created a solver-independent MiniZinc model with representative jobs per batch based on the OPL-model presented in Section~\ref{sec:model-repr-job}.
We briefly describe these two alternative implementations in the following. The full models are available on our website.\footnote{\url{https://cdlab-artis.dbai.tuwien.ac.at/papers/ovenscheduling_cons/}}

\subsection{Alternative CP model using OPL}
\label{sec:model-opl-basic}

This model is based on the ILP model described in Section~\ref{sec:mip-model}.\footnote{We based the OPL model on our ILP model as it turned out that using CP Optimizer as solver via MiniZinc delivers particularly good results with the ILP model, see the results in Section~\ref{sec:experiments}.} 

We use optional interval variables for batches: $B_{m,b}  \subseteq [0,l] $ for all $m \in \mathcal{M}$ and all possible batch positions $b \in [n]$. 
Batches are optional since not all $k \cdot n$ batches will actually be used.
Setup times between batches are also modelled using optional interval variables:  $st_{m,b}  \subseteq [0,l] $ for all $m \in \mathcal{M}$ and all possible batch positions $b \in [n]$. 

Besides these interval variables, we use the same decision variables as in Section~\ref{sec:mip-model}: $X_{m,b,j} \in {0,1}$ to encode whether job $j$ is assigned to batch $B_{m,b}$, $A_{m,b} \in \left[0,a\right]$ for the attribute of batch $B_{m,b}$ and $sc_{m,b}$ for setup costs.

The following constraints are used for the batch and setup time interval variables:
\begin{itemize}
    \item \texttt{presenceOf} ensures that a batch is present iff some job is assigned to it. Similarly, setup times are present iff the preceding and following batch are present.
    \item \texttt{endBeforeStart} is used to enforce the order of batches and setup times on the same machine and \texttt{endAtStart} is used to schedule setup times exactly before the following batch.
    \item \texttt{startOf} restricts the start time of batches to be after the earliest start date of any assigned job and \texttt{lengthOf} is used to enforce that batch processing times lie between the minimal and maximal processing time for every assigned job
    \item \texttt{noOverlap} is used as a redundant constraint to ensure that batches on the same machine do not overlap
    \item \texttt{forbidExtent} is used to guarantee that batches and setup times are scheduled entirely within one machine availability interval: whenever $B_{m,b}$ or $st_{m,b}$ is present, it cannot overlap a point $t$ where $av_m(t)=0$.
\end{itemize}

\subsection{Alternative CP model with representative jobs per batch}
\label{sec:cp-repr}

This model is based on the OPL model presented in detail in Section~\ref{sec:model-repr-job} and is implemented using the high-level constraint modeling language Mini\-Zinc~\cite{nethercote_minizinc_2007}.

In MiniZinc, interval variables are not available. Instead, we model (optional) start times of jobs and processing times of jobs in separate decision variables.
We define the following decision variables: 
\begin{itemize}
\item Optional start time of jobs and optional start times of jobs on machines:   \\ $\text{start}_{j} \in [0, l]$ and $\text{startM}_{j,m} \in [0, l]\quad \forall j \in \mathcal{J} \quad \forall m \in \mathcal{M}$. (this variable is absent for jobs which are not representative or not assigned to the given machine).
\item Processing times of jobs and processing times of jobs on machines: \\ $\text{proc}_{j} \in [0, max_T] $ and $\text{procM}_{j} \in [0, max_T] \quad \forall j \in \mathcal{J} \quad \forall m \in \mathcal{M}$. 
\item Optional pointers to representative jobs in the same batch: \\ $\text{inBatchWith}_j \in \mathcal{J} \quad \forall j \in \mathcal{J}$
\item Optional variables for the length of setup times before batches:\\ $\text{setup}_{j,m} \in [0, max_{ST}] \quad \forall j \in \mathcal{J} \quad \forall m \in \mathcal{M}$.
\end{itemize}
\noindent
Moreover, we use the following auxiliary variables to handle machine availability intervals and to define setup times between batches:
\begin{itemize}
\item Optional variables for the availability interval a job is scheduled to an a given machine: 
    $I_{j,m}  \in [I] \quad  \forall j \in \mathcal{J} \quad \forall m \in \mathcal{M}$.
\item  Attribute of previous job on the same machine:
    $\text{attPrev}_{j,m} \in [0,a] \quad  \forall j \in \mathcal{J} \quad \forall m \in \mathcal{M}$ (is equal to $0$ if there is no previous job).
\item  End time of previous job on the same machine:
    $\text{endPrev}_{j,m} \in [0,l] \quad  \forall j \in \mathcal{J} \quad \forall m \in \mathcal{M}$ (is equal to $0$ if there is no previous job).
\item Binary variables encoding whether a job is the first one on its assigned machine: 
$\text{first}_{j,m} \in \{0,1\} \quad  \forall j \in \mathcal{J} \quad \forall m \in \mathcal{M}$.
\end{itemize}
We make use of the following global constraints:
\begin{itemize}
    \item \texttt{alternative} to ensure that every job that is present, i.e., every representative job, is scheduled to exactly one machine,
    \item \texttt{disjunctive} to enforce that representative jobs on the same machine do not overlap.
\end{itemize}

\subsection{Alternative CP model using state functions in OPL}

Another way of modeling batches in CP Optimizer is with the help of so-called \textit{state function variables}~\cite{laborie2018ibm}, as done, e.g., by Ham et al.~\cite{ham2016flexible} for a flexible job-shop scheduling problem with parallel batch processing machines.
These functions model the time evolution of a value that can be required by interval variables; for the OSP, state functions could model the attribute of jobs.
Two interval variables requiring incompatible states cannot
overlap and interval variables requiring compatible states can be batched together. That is, state functions can be used to model that jobs can only be processed together in a batch if they have the same attribute.
We developed an alternative OPL model for the OSP using state functions; it is available on our website.\footnote{\url{https://cdlab-artis.dbai.tuwien.ac.at/papers/ovenscheduling_cons/}}
However, preliminary experiments showed that this model is not competitive with the one using representative jobs presented in Section~\ref{sec:model-repr-job}.
We therefore did not include it in the experimental evaluation presented in this paper.

\subsection{Programmed Search Strategies}
\label{sec:search}

For our solver-independent MiniZinc models, we implemented several programmed search strategies, which are based on variable- and value selection heuristics. Such heuristics determine the order of the explored variable and value assignments for a CP solver and can play a critical role in reducing the search space that needs to be enumerated by the solver.
For our experiments, we implemented the search strategies directly in the MiniZinc language using search annotations.

\noindent
\textbf{Variable Ordering:}
In our implemented search strategies we select at first an auxiliary variable that captures the total number of batches. For this variable we always use a minimum value first heuristic to encourage the solver to look for low cost solutions early in the search.
Afterwards, we sequentially select decision variables related to a job by assigning the associated batch, machine, batch start time, and batch duration for the job\\ 
(i.e., \(B_{1}, M_{1}, S_{(M_1,B_1)}, P_{(M_1,B_1)},\ldots, B_{\vert\mathcal{J}\vert}, M_{\vert\mathcal{J}\vert}, S_{(M_{\vert\mathcal{J}\vert},B_{\vert\mathcal{J}\vert})}, P_{(M_{\vert\mathcal{J}\vert},B_{\vert\mathcal{J}\vert})}\)).
For the CP model using representative jobs briefly presented in Section~\ref{sec:cp-repr}, we select the jobs' start times and processing times
(i.e., \( \text{start}_{1}, \text{proc}_{1},\ldots, \text{start}_{\vert\mathcal{J}\vert}, \text{proc}_{\vert\mathcal{J}\vert} \)).

\noindent
\textbf{Variable Selection Heuristics:}
We use three different variable selection strategies on the set of decision variables that are related to job assignments: \emph{input order} (select variables based on the specified order), \emph{smallest} (select variables that have the smallest values in their domain first, break ties by the specified order), and \emph{first fail} (select variables that have the smallest domains first, break ties by the specified order).

%\begin{itemize}
%\item \emph{input order}: Select variables based on the specified order 
%\item \emph{smallest}: Select variables that have the smallest values in their domain first (break ties by the specified order)
%\item \emph{first fail}: Select variables that have the smallest domains first (break ties by the specified order)
%\end{itemize}

\noindent
\textbf{Value Selection Heuristics:}
We experimented with two different value selection heuristics for the set of variables which is related to job assignments: \emph{min}~(the smallest value from a variable domain is assigned first), and \emph{split}~(the variable domain is bisected to first exclude the upper half of the domain).
 
 %\begin{itemize}
  % \item \emph{min}: The smallest value from a variable domain is assigned first.
  % \item \emph{split}: The variable domain is bisected to first exclude the upper half of the domain.
 %\end{itemize}

\noindent
\textbf{Evaluated Search Strategies:} Using the previously defined heuristics we evaluated 8 different programmed search strategies:
\begin{enumerate}
    \item \emph{default}: Use the solver's default search strategy.
    \item \emph{search1}: Assign number of batches first, then continue with the solver's default strategy.
    \item \emph{search2}: Assign number of batches first, then continue with \emph{input order} and \emph{min} value selection on the job variables.
    \item \emph{search3}: Assign number of batches first, then continue with \emph{smallest} and \emph{min} value selection on the job variables.
    \item \emph{search4}: Assign number of batches first, then continue with \emph{first fail} and \emph{min} value selection on the job variables.
    \item \emph{search5}: Assign number of batches first, then continue with \emph{input order} and \emph{split} value selection on the job variables.
    \item \emph{search6}: Assign number of batches first, then continue with \emph{smallest} and \emph{split} value selection on the job variables.
    \item \emph{search7}: Assign number of batches first, then continue with \emph{first fail} and \emph{split} value selection on the job variables.
\end{enumerate}

\section{Lower bounds on the optimum}
\label{sec:lower-bounds}

Providing lower bounds on the value of the objective function for an optimal solution can be very helpful when assessing the quality of a solution to the OSP. 
Moreover, in practice, this can aid heuristic search strategies by including lower bounds in a stopping criterion.
In this section, we describe a procedure to calculate such lower bounds. 

First, we will derive lower bounds on the number of batches required in any feasible solution. 
At the same time, bounds on the cumulative processing time of batches are derived.
Then, these lower bounds will be used to calculate lower bounds on the cumulative setup costs. 
Finally, we briefly comment on the number of tardy jobs.
The calculation of these bounds is exemplified with the help of a randomly created instance consisting of 10 jobs at the end of this section.

\subsection{Minimum number of batches required and minimal cumulative batch processing time}
\label{sec:lower-bounds-batch-number}

Whether two jobs may be combined in a batch depends on various properties: their attributes, their respective minimal and maximal processing times, their respective eligible machines as well as their sizes and the machine capacities. We can use these properties in order to obtain bounds on the minimum number of batches required in any feasible solution.

Since jobs can only be combined if they share the same attribute, bounds on the number of batches required can be calculated independently for all attributes and then added up.
In the following, we thus determine lower bounds on the number of batches needed for jobs of some fixed attribute $r \in 
\mathcal{A}$ and denote by $b_r$ the number of batches with attribute $r$. Similarly, $p_r$ denotes the minimal cumulative processing time of batches with attribute $r$.

\paragraph{Bound based on the maximal machine capacity.}

Due to the capacity constraints of machines, a simple bound on the number of batches required is
\begin{equation}
 b_r \geq \left\lceil  \frac{\sum_{j \in \mathcal{J}: a_j = r}s_j}{\max_{m \in \mathcal{M}}\{c_m\}} \right \rceil.   
 \label{eqn:lower-bound-machine-cap}
\end{equation}
This corresponds to the minimal number of batches required if we assume that jobs can be split into smaller jobs of unit size and that all jobs can be scheduled on the machine with largest machine capacity.

\paragraph{Refinement of the bound based on a distinction between large and small jobs.}
The bound in equation~\eqref{eqn:lower-bound-machine-cap} can be refined by dividing the set of jobs with attribute $r$ into ``large'' and ``small'' jobs in a similar fashion as Damodaran et al.~\cite{damodaran2012simulated}.
Let us denote by $J_r^l$ the set of jobs with attribute $r$ that fulfill the following condition
\[
J_r^l = \left\lbrace j \in \mathcal{J}: a_j = r \wedge \max_{m \in \mathcal{E}_j}\{c_m\} - s_j < \min_{i \in \mathcal{J}: a_i =r}{s_i}\right\rbrace,
\]
i.e., these are ``large'' jobs that cannot accommodate any other jobs in the same batch even if one chooses to process them on an eligible machine with maximal capacity. All jobs in $J_r^l$ thus need to be processed in a batch of their own.
The set of ``small'' jobs $J_r^s$ is defined as the set of jobs with attribute $r$ that are not in $J_r^l$. For these jobs, we can use the previously established lower bound in equation~\eqref{eqn:lower-bound-machine-cap}. This leads to
\begin{equation}
b_r \geq \vert J_r^l \vert + \left\lceil  \frac{\sum_{j \in J_r^s}s_j}{\max_{m \in \mathcal{M}}\{c_m\}} \right \rceil.
     \label{eqn:better-lower-bound-machine-cap}
\end{equation}

\paragraph{Refinement of the bound for small jobs based on machine eligibility.}

The lower bound for small jobs can be further improved by taking into account the eligibility of machines; we will denote this bound on the number of batches by $b_r^E$ (``E'' for eligible machines). Indeed, some jobs might only be executable on a single machine. This implies that batches on these machines must be created even if there are batches on other machines with free capacity. In order to take the jobs' eligible machines into account, let us introduce the following notation:
\[
b_{r,i} = \frac{
\sum_%{\substack
{j \in J_r^s:  \mathcal{E}_j=\{i\}}%}
s_j}
{c_i},
\]
i.e., $\lceil b_{r,i} \rceil$ is the minimal number of batches consisting of small jobs that need to be processed on machine $i$ for $i \in \mathcal{M}.$
Moreover, let us denote by $cap_i$ the total remaining 
capacity in the batches created for machine
$i$ after all small jobs that need to be processed on this machine have been scheduled:
\[
cap_i = ( \lceil b_{r,i} \rceil -b_{r,i} ) \cdot c_i
\]
Furthermore, we denote by $S_r$ the sum of job sizes for all small jobs of attribute $r$ that have more than one eligible machine:
\[
S_r = \sum_{j \in J_r^s: \vert \mathcal{E}_j \vert > 1}s_j.
\]
The bound in equation~\eqref{eqn:better-lower-bound-machine-cap} can then be improved as follows:
\begin{align}
b_r & \geq \vert J_r^l \vert + b_r^E \notag\\
&= \vert J_r^l \vert + \sum_{i \in \mathcal{M}} \lceil b_{r,i} \rceil + \left\lceil \frac{\max{(0,S_r - \sum_{i \in \mathcal{M}}{cap_i})}}{\max_{m \in \mathcal{M}}\{c_m\}} \right\rceil
    \label{eqn:even-better-lower-bound-small-jobs-elig-machines}
\end{align}
All small jobs that need to be processed on a specific machine are scheduled on this machine. Then, the remaining small jobs are used to fill up these batches. If there are still jobs left, we assume that they can be scheduled on the machine with maximal capacity. The minimum number of batches required for these jobs corresponds to the last summand in the equation above. Note that the bound $b_r^E$ could be tightened further by checking if there are larger disjoint subsets of eligible machines among the set of jobs $J_r^s$.

Let us now turn to the minimal processing time of the batches we just created. 
For the batches consisting of a large job, the processing time is simply equal to the respective minimal job processing time.
For the small jobs $J_r^s$, we can create the list of minimal batch processing times $\texttt{times}_{r}$ in the following fashion.
For every machine $i \in \mathcal{M}$ for which $\lceil b_{r,i} \rceil > 0$, proceed as follows:
\begin{enumerate}
    \item Create the list of minimal job processing times $mint_j$ of all small jobs that need to be processed on machine $i$.
    \item Sort this list in increasing order.
    \item For the first batch, take the last element in the list, i.e., the largest minimal processing time  (The job with this processing time needs to be processed in some batch and the batch must have this processing time).
    Add this element to $\texttt{times}_{r}$.
    \item For the remaining  batches, pick the first $\lceil b_{r,i} \rceil - 1$ processing times from the list and add them to $\texttt{times}_{r}$.
\end{enumerate}
Finally, for the additional $b_{r,*} = \left\lceil \max{(0,S_r - \sum_{i \in \mathcal{M}}{cap_i})}/ \max_{m \in \mathcal{M}}\{c_m\} \right\rceil$ batches that need to be formed, proceed as follows:
\begin{enumerate}
    \item Create the list of minimal job processing times of all small jobs that can be processed on multiple machines and sort it in increasing order.
    \item  Denote by $\texttt{max}$ the maximum element of this list. Check if $\texttt{max} > \max(\texttt{times}_{r})$:
    \begin{enumerate}
        \item If yes, remove $\max(\texttt{times}_{r})$ from $\texttt{times}_{r}$ and replace it by $\texttt{max}$. (Again, the job with this processing time needs to be processed in some batch). This step can be done even if $b_{r,*} = 0$.
        Then, pick the first $(b_{r,*}-1)$ many elements from the list and add them to $\texttt{times}_{r}$.
        \item If not, pick the first $b_{r,*}$ many elements from the list and add them to $\texttt{times}_{r}$.

    \end{enumerate}
\end{enumerate}
The sum of the elements of the list $\texttt{times}_{r}$ is then a lower bound on the cumulative batch processing time for small jobs of attribute $r$. We denote this number by $p_r^E$.

Overall, this leads to the following lower bound for the cumulative batch processing time
\begin{equation}
    p_r = \sum_{j \in J_r^l}mint_j + p_r^E,
    \label{eqn:lower-bound-runtime-elig-machi}
\end{equation}
where $p_r^E$ is calculated according to the procedure described above.

\paragraph{Alternative refinement of the bound for small jobs based on compatible job processing times.}

In the bound in equation~\eqref{eqn:even-better-lower-bound-small-jobs-elig-machines} we have ignored that jobs can only be processed in the same batch if they have compatible processing times. Indeed, 
two jobs $i$ and $j$ with respective minimal and maximal processing times $mint_i, mint_j$ and $maxt_i, maxt_j$
may be combined in a batch if the intervals of their possible processing times have a non-empty intersection:
\begin{equation}
[mint_i, maxt_i] \cap [mint_j, maxt_j] \neq \emptyset.
\label{eqn:comp-proc-time}
\end{equation}
In the following, we will see how this compatibility requirement can be used to obtain an alternative lower bound on the number of batches and the minimal processing time for small jobs.
For this purpose, let us consider the following special case of the OSP:
\begin{quote}
OSP*: Given a set of jobs $\mathcal{I}$ of unit size defined by their minimal and maximal processing times, i.e. $j=[mint_j, maxt_j]$ for all $j \in\mathcal{I}$, and a single machine with capacity $c \in \mathbb{N}$, how many batches do we need at least in order to process all jobs if jobs can only be processed in the same batch if the compatibility condition~\eqref{eqn:comp-proc-time} is fulfilled?\end{quote}
Several variants of this problem have been studied in the literature, e.g.\ by Finke et al.~\cite{finke2008batch}; the variant that we are interested in corresponds to problem (P2) there.
Solving the problem OSP* will allow us to obtain lower bounds for the OSP: Indeed, as in equation~\eqref{eqn:lower-bound-machine-cap}, we obtain lower bounds on the number of batches required and their processing times if we assume that jobs can be split into smaller jobs of unit size and that all jobs can be scheduled on the machine with largest machine capacity.

The compatibility relation between jobs defined in equation~\eqref{eqn:comp-proc-time} can be represented with the help of a \textit{compatibility graph} $G=(V,E)$, where $V$ is the set of all jobs $\mathcal{I}$ and $(i,j) \in E$ if and only if the jobs $i$ and $j$ have compatible processing times, i.e., the condition in equation~\eqref{eqn:comp-proc-time} is met.
In this graph, a batch forms a (not necessarily maximal) clique.
The problem of solving OSP* is thus equivalent to covering the nodes of the compatibility graph with the smallest number of cliques with size no larger than $c$.

This problem is \NP-complete for arbitrary graphs, but solvable in polynomial time for interval graphs. A simple greedy algorithm is provided  by Finke et al.~\cite{finke2008batch} and referred to as the algorithm GAC (greedy algorithm with compatibility).
By adapting the order in which jobs are processed by the GAC algorithm, we obtain an algorithm that minimizes both the number of batches and the cumulative batch processing time.
We call this algorithm GAC+.

\begin{namedtheorem}[The Algorithm GAC+]
Consider the jobs in non-increasing order $j_1, j_2, \ldots, j_n$  of their minimal processing times $mint_j$, breaking ties arbitrarily.
Construct one batch per iteration until all jobs have been placed into batches.
In iteration $i$, open a new batch $B_i$ and label it with the first job $j^*$ that has not yet been placed in a batch. Starting with $j^*=[mint_{j^*},maxt_{j^*}]$, place into $B_i$ the first $c$ not yet scheduled jobs $j$ for which $mint_{j^*} \in [mint_{j},maxt_{j}]$ (or all of them if there are fewer than $c$). 
\end{namedtheorem}

For a set $\mathcal{I}$ of unit size jobs and a maximum batch size $c \in \mathbb{N}$, we denote by $GACb(\mathcal{I},c)$ the number of batches returned by the GAC+ algorithm above.  Similarly, let $GACp(\mathcal{I},c)$ denote the minimal processing time returned by the GAC+ algorithm for this instance.

\begin{theorem}
For any given set of unit size jobs $\mathcal{I}$ and for any given constant $c \in \mathbb{N}$, Algorithm GAC+ solves the problem OSP*, i.e., $GACb(\mathcal{I},c)$ is the minimum number of batches required under the condition that a batch may not contain more than $c$ jobs. Moreover, the cumulative batch processing time $GACp(\mathcal{I},c)$ is minimal.
\label{thm:algo-GAC+}
\end{theorem}

By slight abuse of notation, for a set $\mathcal{J}$ of jobs with arbitrary job sizes, let $GACb(\mathcal{J},c)$ denote the number of batches returned by the GAC+ algorithm when replacing every job $j \in \mathcal{J}$ with $s_j$ identical copies of unit size jobs. Similarly, let $GACp(\mathcal{J},c)$ denote the minimal processing time returned by the GAC+ algorithm for this instance.
With this notation, Theorem~\ref{thm:algo-GAC+} yields the following bounds:
\begin{align}
b_r & \geq \vert J_r^l \vert +  b_r^C, &\text{ with } b_r^C = GACb(J_r^s,\max_{m \in \mathcal{M}}\{c_m\}) 
    \label{eqn:even-better-lower-bound-small-jobs-proc-times} \\
p_r & \geq \sum_{j \in J_r^l}mint_j + p_r^C, &\text{ with } p_r^C = GACp(J_r^s,\max_{m \in \mathcal{M}}\{c_m\}),
\label{eqn:even-better-lower-bound-runtime-small-jobs-proc-times}
\end{align}

\begin{proof}[Proof of Theorem~\ref{thm:algo-GAC+}]
We follow the proof of Theorem~4 in~\cite{finke2008batch}, extending it to include the minimization of the cumulative batch processing time and adapting it to our variant of the algorithm.
The proof is by induction over the number of jobs and the induction start with a single job is trivial.

Let us start with a simple observation about the minimum number of batches and the minimal batch processing time.
For this, let $b(\mathcal{I},c)$ denote the minimum number of  batches required to schedule all jobs in $\mathcal{I}$ under the condition that a batch may not contain more than $c$ jobs. Similarly, let $p(\mathcal{I},c)$ denote the minimal cumulative batch processing time in any schedule of all jobs in $\mathcal{I}$.
Then these two functions are monotonous in $\mathcal{I}$, i.e.:
\begin{align}
\begin{split}
&b(\mathcal{I},c) \geq b(\mathcal{I} \setminus \{j \},c) \\
\text{and } &p(\mathcal{I},c) \geq p(\mathcal{I} \setminus \{j \},c), \text{ for every } j \in \mathcal{I}. \label{eqn:min-batch-monotone}
\end{split}
\end{align}

For the induction step, let $\mathcal{B} = (B_1, B_2, \ldots, B_b)$ be a sequence of batches constructed by the algorithm GAC+ for $\mathcal{I}$, $B_1$ being the first batch constructed by the algorithm and $p \in \mathbb{N}$ being the cumulative batch processing time of $\mathcal{B}$.
Let the label of $B_1$ be the job $i = [mint_i, maxt_i]$, i.e., $mint_i$ is maximal among the minimal processing times and the processing time of $B_1$ is equal to $mint_i$.
For the set of jobs $\mathcal{I} \setminus B_1$, the algorithm constructs the batch sequence $B_2, \ldots, B_b$ (see the definition of GAC+).
By the induction hypothesis we know that $B_2, \ldots, B_b$ is optimal for $\mathcal{I} \setminus B_1$, i.e., $b(\mathcal{I} \setminus B_1,c) = \vert \mathcal{B} \vert -1 = b-1$ and $p(\mathcal{I} \setminus B_1,c) = p - mint_i$.
It thus suffices to show that there exists a batch sequence of minimal length and with minimal batch processing time that contains the batch $B_1$.

Let $O_1$ be the batch containing $i$ in an optimal sequence of batches $\mathcal{O}$ and let us choose $\mathcal{O}$ such that the size of the intersection $\vert O_1 \cap B_1 \vert$ is maximal. We will prove that $O_1 = B_1$.

First note that $i \in O_1$ implies that $mint_i \in [mint_j, maxt_j]$ for all jobs $j \in O_1$: $mint_j \leq mint_i$ since $mint_i$ is maximal and $mint_i \leq maxt_j$ since every job $j \in O_1$ needs to be compatible with $i$. Thus the processing time of batch $O_1$ is equal to $mint_i$.

We now distinguish two cases: $\vert B_1 \vert < c$ and $\vert B_1 \vert= c$, where $c$ is the maximum batch size. If $\vert B_1 \vert< c$, the batch $B_1$ contains all neighbors of $i$ in the compatibility graph $G$ corresponding to the set of jobs $\mathcal{I}$. Since $O_1$ is a clique containing $i$, it follows that  $O_1 \subseteq B_1$.
Then by monotonicity (as stated in equation~\eqref{eqn:min-batch-monotone}), we have
\begin{align*}
\vert \mathcal{B} \vert -1 & = b(\mathcal{I} \setminus B_1,c) \leq b(\mathcal{I} \setminus O_1,c) = \vert \mathcal{O} \vert -1 \\
\text{ and } p - mint_1 & = p(\mathcal{I} \setminus B_1,c) \leq p(\mathcal{I} \setminus O_1,c) = p(\mathcal{I},c) -mint_i ,
\end{align*}
which proves that $\mathcal{B}$ is optimal both in terms of the number of batches and in terms of the cumulative processing time.

For the case $\vert B_1 \vert= c$, we assume towards a contradiction that there exists a job $j=[mint_j, maxt_j] \in B_1 \setminus O_1$.
This implies that there must also exist a job $k=[mint_k, maxt_k] \in O_1 \setminus B_1$. (As before, $ O_1 \subset B_1$ would imply that $\mathcal{B}$  is optimal. However, the job $j$ could have been added to $O_1$ without having an impact on the number of batches or the batch processing time required by $\mathcal{O}$. This however is a contradiction to the choice of $\mathcal{O}$).
From the definition of the algorithm, we know that $B_1$ consists of the first $c$ jobs containing $mint_i$. Therefore, $j < k$ and $mint_j \geq mint_k$. Moreover, as noted earlier, we know that $mint_i \in k=[mint_k, maxt_k]$ and thus $[mint_j,mint_i] \subseteq k$.

We then define $O_1' \coloneqq (O_1 \setminus \{ k \}) \cup \{j\}$ and redefine the batch $O \in \mathcal{O}$ that contains $j$ as $O' \coloneqq (O \setminus \{ j \}) \cup \{k\}$.
Both these batches fulfill the compatibility constraint for the processing times: $O_1'$ does because $mint_i$ is contained in $j$ and in all jobs in $O_1$ and $O'$ does because all jobs that are compatible with $j$ are also compatible with $k$ (If job $s$ is compatible with $j$, 
this means that $s=[mint_s, maxt_s] \cap [mint_j,mint_i] \neq \emptyset$, since $mint_i$ is maximal among all minimal processing times. On the other hand, we already noted that 
$[mint_j,mint_i] \subseteq k$ and thus $s \cap k \neq \emptyset,$ which means that $s$ and $k$ are compatible.)
As for the processing times of the batches, both batches $O_1$ and $O_1'$ have the processing time $mint_i$ as they contain job $i$.
For the batch $O'$, we have replaced the job $i$ with a job with smaller or equal processing time ($mint_i \geq mint_k$). Thus the processing time of batch $O'$ is smaller or equal to the batch processing time of $O$.
We have thus produced another optimal sequence of batches $\mathcal{O}'= \mathcal{O} \setminus \{ O_1, O\} \cup \{O_1', O'
\}$.
However, $\vert  O_1'\cap B_1 \vert > \vert  O_1\cap B_1 \vert$ which is in contradiction to the choice of $\mathcal{O}$. This finishes the proof. 
\end{proof}

\paragraph{Best overall bound on the number of batches and the minimal processing time}

Finally, for an overall bound on the number of batches and the minimal processing time, we compute the bounds based on the eligibility of machines as in equation~\eqref{eqn:even-better-lower-bound-small-jobs-elig-machines} and the bounds based on the compatibility of processing times as in equation~\eqref{eqn:even-better-lower-bound-small-jobs-proc-times}, take the maximum of these two values for every attribute and sum up these numbers:
\begin{align}
b & \geq \sum_{r=1}^a (\vert J_r^l \vert + \max(b_r^E, b_r^C)) \label{eqn:best-lower-bound-batch-number} \\
p & \geq \sum_{r=1}^a(\sum_{j \in J_r^l}mint_j + \max(p_r^E, p_r^C)),
\label{eqn:best-lower-bound-proc-time}
\end{align}

\subsection{Bounds on the other components of the objective function}
\label{sec:lower-bounds-obj-components}

\paragraph{Setup costs}

A simple lower bound on the setup costs is
\begin{equation}
sc \geq \sum_{r=1}^a b_r \cdot \min_{s \in \{1, \ldots, a \}}\{ sc(s,r) \}, 
\label{eqn:lower-bound-setup-costs-before-batch}
\end{equation}
i.e., we assume that the setup costs \textit{before} batches are always minimal.
Note that $\min_{s}\{ sc(s,r) \}$ is the minimum value in the $r$-th column of the setup cost-matrix.

A similar bound can be derived assuming that the setup costs \textit{after} batches are always minimal. For this case, we need to include initial setup costs for all machines to which batches are scheduled but may not include setup costs after the last batch on every machine.
In order to do so, we create the list \texttt{setup\_costs}. For every attribute $r$, we add $b_r$ copies of $\min_{s \in \{1, \ldots, a \}}\{ sc(r,s) \}$ to \texttt{setup\_costs} (this is the minimum value in the $r$-th row of the setup cost-matrix). Moreover, for every machine $m$, we add the element $\min_{s \in \{1, \ldots, a \}}\{ sc(s_m,s) \}$ to \texttt{setup\_costs}.
This list thus consists of $\sum_r b_r + m = b + m$ elements.
The list is then sorted increasingly and the sum of the first $b$ elements is taken. %; this corresponds to the fact that there are no setup costs after the final batch on every machine (or, for the case of a machine with no batches, there are no setup costs at all). 
We thus have: 
\begin{equation}
sc \geq \sum_{i = 1}^b \texttt{setup\_costs}(i).
\label{eqn:lower-bound-setup-costs-after-batch}
\end{equation}
Altogether, we have the following lower bound on the setup costs
\begin{equation}
    sc \geq \max
    \left( \sum_{r=1}^a b_r \cdot \min_{s \in \{1, \ldots, a \}}\{ sc(s,r) \}, 
    \sum_{i = 1}^b \texttt{setup\_costs}(i)
    \right).
    \label{eqn:lower-bound-setup-costs}
\end{equation}

Note that it is not possible to obtain a lower bound on the setup costs for any feasible solution by computing the minimum setup costs for a solution that uses the minimum number of batches per attribute calculated previously. Indeed, if the matrix of setup costs does not fulfill the triangle inequality, it can be advantageous to introduce additional batches if the sole goal is to reduce setup costs. Although this scenario does not make much sense in a practical setting, we cannot safely exclude it for our randomly generated instances.
The bound derived in equation~\eqref{eqn:lower-bound-setup-costs} is nonetheless an actual lower bound on any feasible solution for the given problem instance; it is not possible to achieve smaller cumulative setup costs by adding additional batches since we have assumed minimal setup costs before and after every batch.

\paragraph{Number of tardy jobs.}
As for the number of tardy jobs, we cannot infer that much from the instance itself. 
What we can do, is to schedule every job independently in a batch of its own on the first eligible machine that is available and to compute the resulting completion time. If the completion time is after the job's latest end date, this will be a tardy job in every solution.
By computing the number of jobs that finish late in this way, we obtain a lower bound on the number of tardy jobs in any solution.

\subsection{Example}

Let us now consider the following randomly created instance consisting of 10 jobs, 2 machines and 2 attributes in order to exemplify the calculation of the bounds on the objective function that we just derived.
The parameters for the generation of the random instance were chosen in such a way that the job sizes and machine capacities do not allow for much batching (see Section~\ref{sec:random-instances} for the description of our random instance generator).

\begin{center}
{
\begin{tabular}{p{2.5cm}|p{2cm}p{2cm}}
Machine & $M_1$ & $M_2$   \\\hline
$c_m$ & 18 & 20 \\\hline
$s_m$ & 1 & 2 \\\hline
Availability \mbox{intervals} $[as, ae]$ & [21,250]  & [103,259] \end{tabular}
\quad
$
% st = \begin{pmatrix}
% 0 & 0 \\
% 3 & 8
% \end{pmatrix} \quad
sc = \begin{pmatrix}
6 & 8  \\
10 & 10
\end{pmatrix}
$
}

\small{
\begin{tabular}{l|cccccccccc}
Job $j$ & \textbf{1} & \textbf{2} & \textbf{3} & 4 & 5 & \textbf{6} & 7 & 8 & 9 & 10 \\\hline
$\mathcal{E}_j$ & $M_1$ & $M_1$ &  & $M_1$ & $M_1$ & & $M_1$ & $M_1$ & &$M_1$\\
& $M_2$ & $M_2$   & $M_2$ & &$M_2$ & $M_2$ & $M_2$& &$M_2$ &$M_2$ \\\hline
$et_j$ & 2 & 3 &8&1&39&41&40&31&27&16\\
$lt_j$  &  16&20&43&24&55&64&56&89&58&27 \\
$mint_j$  & 11&10&19&19&10&19&11&50&19&11\\
$maxt_j$    &  11&50&19&19&50&50&50&50&19&50\\
$s_j$     &  \textbf{18}&\textbf{16}&\textbf{17}&2&6&\textbf{19}&11&11&4&14\\
$a_j$      &  2&2&2&1&2&2&2&2&1&1\\
\end{tabular}
}
\end{center}

\begin{table}[]
\begin{center}
\small{
\begin{tabular}{l|ccc || ccc }
& \multicolumn{3}{c||}{number of batches} & \multicolumn{3}{c}{minimal processing time} \\
& \eqref{eqn:lower-bound-machine-cap} 
& $b_r^E$\eqref{eqn:even-better-lower-bound-small-jobs-elig-machines}
& $b_r^C$\eqref{eqn:even-better-lower-bound-small-jobs-proc-times}
 & large jobs & $p_r^E$ \eqref{eqn:lower-bound-runtime-elig-machi} & $p_r^C$ \eqref{eqn:even-better-lower-bound-runtime-small-jobs-proc-times}\\\hline
attribute 1 & \multirow{2}{0.5cm}{6}  &  2 & 1  & 0 & 38 & 19 \\ 
attribute 2 & & 6 & 6  & 59 & 60 & 61\\ \hline
sum & \multicolumn{3}{c||}{8} & \multicolumn{3}{c}{158}
    \end{tabular}}
\end{center}
    \caption{Lower bounds for the number of batches and cumulative batch processing times for an example instance with 10 jobs.}
    \label{tab:ex-lower-bounds}
\end{table}
    
The values of the lower bounds for the number of batches required and the cumulative batch processing times are summarized in Table~\ref{tab:ex-lower-bounds}. We explain their calculation in what follows.
The sets of large jobs are
\[
J_1^l = \emptyset \text{ and } J_2^l = \{ 1, 2, 3, 6\}  ,
\]
we thus need 4 batches for the large jobs of attribute 2 and none for attribute 1. The large jobs and their sizes are written in bold in the table above. The processing times for large batches are given by the minimal processing times of the large jobs and contribute $11 + 10 + 19 + 19 = 59$ to the cumulative batch processing time.

For the processing time of small jobs, we exemplify the calculation of the bound based on eligible machines for attribute 1 and the one of the bound based on compatible processing times for attribute 2.
For attribute 1, we have three small jobs (4, 9 and 10) of which job 4 can only be processed on machine 1 and job 9 only on machine 2. Two different batches are thus required for these jobs.
Since the cumulative remaining machine capacity ($2\cdot \max\{c_m\}-(s_4+s_9)=40-(2+4)=34$) is sufficient to accommodate job 10 with $s_{10}=14$, these two batches suffice.
In this case, the runtime of the two batches is simply given by the minimal runtime of the two jobs 4 and 9 and is equal to 38 in total.
As for attribute 2, the small jobs are: 5, 7 and 8. Their respective intervals of possible processing times are $[10,50]$, $[11,50]$ and $[50,50]$. In order to follow algorithm GAC+, we sort the list of jobs in decreasing order of their minimal processing times: $(8, 7, 5)$.
A first batch with processing time $50$ is created for job $8$. The remaining capacity in this batch is $20-11=9$ (assuming that is assigned to the batch with maximal capacity).
We thus proceed in the list of jobs and add 9 of the 11 units of job 7 to this batch.
For the remaining 2 units of job 7, a new batch with processing time $11$ is created. We can add the entire job 5 to this batch.
In total, two batches with a cumulative processing time of $61$ are needed for the small jobs of attribute 2.

For the calculation of setup costs, equation~\eqref{eqn:lower-bound-setup-costs-before-batch} gives:
\begin{align*}
sc &\geq b_1 \cdot \min_{s}\{ sc(s,1) \} + b_2 \cdot \min_{s}\{ sc(s,2) \} \\
& = 2 \cdot \min(6,10) + 6 \cdot \min(8,10)= 60.
\end{align*}
For equation~\eqref{eqn:lower-bound-setup-costs-after-batch}, the list of minimal setup costs \texttt{setup\_costs} contains $\min_{s}\{ sc(1,s) \}=\min(6,8)$ three times (twice for attribute 1 and once for the initial state of machine 1) and $\min_{s}\{ sc(2,s) \}=\min(10,10)$ seven times (six times for attribute 2 and once for the initial state of machine 2). We take the $b=8$ smallest values from this list and thus have:
\[
sc \geq \sum_{i = 1}^8 \texttt{setup\_costs}(i) = 3 \cdot 6 + 5 \cdot 10 = 68.
\]
We can take the maximum of these two values and obtain that $sc \geq 68$ for this instance.

Due to the given machine availability intervals for this instance, all jobs except jobs 5, 7 and 8 always finish late. Thus, the number of tardy jobs is $\geq 7$ in any feasible solution.

Using weights $w_p=4$, $w_{sc}=1$ and $w_t=100$ and aggregating  the lower bounds for three components of the objective function as in~\eqref{eqn:obj-function}, we obtain that:
\[
obj \geq \frac{4 \cdot 158/18 + 68/10 + 100 \cdot 7 }{10 \cdot 105} \approx 0.7066,
\]
since the average minimal processing time is $\lceil 179/10\rceil =18$ and the overall maximal setup cost is $10$.

For this particular instance and the given weights, an optimal solution has 8 batches, a cumulative processing time of $158$, setup costs of $72$ and 8 jobs that finish too late (all except jobs 5 and 7 which are scheduled in the same batch at the second position on machine 1).
The value of the objective function is $obj \approx 0.8002$.
The gap between the calculated lower bounds and the optimal solution are thus $0 \%$ for the runtime, $5.5\%$ for the setup costs and $12.5\%$for the number of tardy jobs; the gap for the aggregated objective function is $11.7\%$ (due to the high weight given to tardy jobs).

\section{Random instance generator and construction heuristic}
\label{sec:random+heuristic}

\subsection{Generation of random instances}
\label{sec:random-instances}

The random instance generator we propose is based on random instance generation procedures for related problems from the literature~\cite{malve2007genetic,gallego2009algorithms}. However, as the existing variants were designed for batch scheduling problems which neither include machine eligibility constraints, machine availability times nor setup costs and times, the random generation of the associated instance parameters is a novel contribution of this paper.
The list of parameters for this instance generator is given in Table~\ref{tab:param_random_instance_generator}.
\begin{table}[ht]
	\centering
        %\scriptsize
        %\footnotesize
        \small
		\begin{tabular}{p{1.6cm}p{8cm}p{4cm}}
		Name &       Description &        Values       \\\toprule
		\multicolumn{3}{l}{Parameters relating to jobs} \\\midrule
		$n$ & number of jobs & 10, 25, 50, 100  \\
		$max_T$ & overall maximum processing time & 10, 100 \\
		%$diff_T$ & number of different processing times &  $n/5$, $n/2$\\
		\texttt{max\_time} & \texttt{true} if jobs have a max. processing time & \texttt{true}, \texttt{false} \\
		$\rho$ & determines spread of earliest start times & 0.1, 0.5 \\
		$\phi$ & determines time from earliest start to latest end of job & 2, 5 \\
		$\sigma$ & determines number of eligible machines per job & 0.2, 0.5 \\
		$s$ & maximum job size & 5, 20 \\
		\midrule
		\multicolumn{3}{l}{Parameters relating to attributes} \\\midrule
		$a$ & number of attributes & 2, 5\\
		\texttt{s\_time} \newline 
		= \texttt{s\_cost} & type of setup-time matrix \newline type of setup-cost matrix  & \texttt{constant}, \texttt{arbitrary}, \texttt{realistic}, \texttt{symmetric} \\
		\midrule
		\multicolumn{3}{l}{Parameters relating to machines} \\\midrule
		$k$ & number of machines & 2, 5\\
		$min_C \newline = s$ & lower bound for max. machine capacity \newline(=max. job size) & 5, 20 \\
		$max_C$ & upper bound for maximum machine capacity & 20, 100 \\
		$\tau$ & lower bound for the fraction of time machines are available & 0.25, 0.75\\
		$max_I$ & max.\ number of availability intervals & 5 \\
	\end{tabular}
			\caption{List of parameters of the random instance generator.}
	\label{tab:param_random_instance_generator}
\end{table}

\noindent
\emph{Jobs.} The list of $n$ jobs is generated as follows. %First, given the overall maximum processing time $max_T$ and the number of different processing times $diff_T$, a list of $diff_T$ many different processing times is created. The values are random integers in the interval $[1,Max_T]$. 
First, for every job $j$, the minimal processing time $mint_j$ is chosen using a discrete uniform distribution $U(1,max_T)$.
For the maximum processing time, there are two options: either there is no upper limit on the processing time of jobs (\texttt{max\_time} = \texttt{false}), in which case the maximum processing time is set to $max_T$ for all jobs. Or, if \texttt{max\_time} = \texttt{true},  the maximum processing time for a job is chosen using a discrete uniform distribution $U(mint_j,max_T)$.
Next, the earliest start and latest end times are determined for every job. The earliest start time $et_j$ is chosen similarly as by Velez Gallego~\cite{gallego2009algorithms} according to a discrete uniform distribution $U(0,\lceil \rho \cdot Z \rceil)$ where  $\rho \in [0,1]$ and $Z = \sum mint_j$ is the total processing time of all jobs. If $\rho =0$, all jobs are available right at the beginning and as $\rho$ grows, the jobs are released over a longer interval. The latest end time $lt_j$ is chosen as in~\cite{malve2007genetic} according to $lt_j =et_j +\lfloor U(1, \phi)\cdot mint_j \rfloor$ where $\phi \geq 1$. If $\phi = 1$, the
latest end time is equal to the sum of the earliest start time and the minimum processing time, meaning that all jobs must be processed immediately in order to finish on time.
As $\phi$ grows, more time is given for every job to be completed and tardy jobs are less likely. 
Regarding the set of eligible machines for a job, one machine is chosen at random among all machines. Additional machines are then added to this set with probability $\sigma$ each.
%, i.e., if $\sigma=0$, every job will only have a single eligible machine and if $\sigma =1$, all machines will be eligible for all jobs.
The size $s_j$ and attribute $a_j$ of a job are both chosen at random between $1$ and the maximum job size $s$ or number of attributes respectively.
% and the attribute is chosen at random between $1$ and $a$, the number of different attributes.

\noindent
\emph{Attributes.} The setup times and setup costs matrices can be of four different types: Constant, arbitrary, realistic and symmetric. For the type ``constant'', setup times/costs are all equal to a randomly chosen constant between 0 and $\lceil max_T/4\rceil$. For the type ``arbitrary'', every entry is chosen independently at random between 1 and $\lceil max_T/4\rceil$. For the type ``realistic'', setup times/costs between two batches of the same attribute are lower and are chosen independently at random between 0 and $\lceil max_T/8\rceil$, whereas setup times/costs between different attributes are higher and are chosen between $\lceil max_T/8\rceil + 1$ and $\lceil max_T/4\rceil$. For the type ``symmetric'' a symmetric matrix is generated with random entries between 0 and $\lceil max_T/4\rceil$.

\noindent
\emph{Machines.} The maximum machine capacity $c_m$ is randomly chosen between $min_C$ and $max_C$ where the lower bound $min_C$ is set to the maximum job size $s$ to ensure that every job fits into every machine.
The initial state $s_m$ is chosen at random among the set of attributes $\{ 1, \ldots, a\}$.
For the machine availability times, we first fix the length of the scheduling horizon $l$.
%Since the purpose of the scheduling horizon is mainly to keep the size of the search space manageable and it should not have an influence on whether an instance is hard or easy to solve, we want to calculate a generous upper bound for the termination of the last job.
If we assume that every job is processed in a batch of its own and that all jobs are processed on the same machine, the total runtime is at most equal to the sum of all processing times $Z$ plus $n$ times the maximal setup time $max_{st}$. The parameter $\tau \in (0,1]$ is a lower bound for the fraction of time that every machine is available. Thus, if $max_{et}$ is the latest earliest start time, all jobs should--on average--be finished at time
\[
l = max_{et} + \left\lceil (Z + n\cdot max_{st}) / (\tau) \right\rceil
\]
which we use to set the length of the scheduling horizon. 
Note that if the latest end date of a job is greater than this upper bound we simply use it instead.
Now, for every one of the $k$ machines, we pick the number of availability intervals $I$ randomly between 1 and $max_I$.
%Recall that every machine's availability times are specified by the start and end times of intervals $[start_i, end_i]$ during which the machine is available; at all other times within the scheduling horizon the machine is not available.
Every interval $[start_i, end_i]$ should be long enough to accommodate at least a single job with minimal processing time $min_T = \min(mint_j: j \in \mathcal{J})$ (plus the necessary setup times).   
Thus, the minimum distance between two interval start times $start_i$ and $start_{i+1}$ is $d = min_T + max_{st}$.
%Next, we pick the start times of availability intervals:
We first pick the start time $start_{1}$ of the first interval: 
In order to guarantee that every machine is available at least a fraction $\tau$ of the time, $0 \leq start_1 \leq \lfloor l \cdot (1-\tau) \rfloor$ must hold and in order to leave enough time for all availability intervals, it has to hold $start_1 \leq l - I \cdot d$.
Next, for the start times of the remaining intervals, we pick $I-1$ random integers between $(start_{1}+d)$ and $(l-d)$ that are at least $d$ apart.
Finally, we determine the end time $end_i$ of the $i$-th interval:
\[
end_i = start_i +\max(d,\lceil U(\tau, 1)\cdot(start_{i+1} - start_i) \rceil)
\]

\subsection{Construction heuristic} 
\label{sec:heuristic}

We designed a construction heuristic that can find initial solutions and can also serve as upper bound on the optimal solution cost.
The heuristic is a \textit{dispatching rule} that prioritizes jobs first by their earliest start date and then by their latest end date.
Similar rules, such as the Earliest Due Date (EDD) dispatching rule presented by Uzsoy~\cite{uzsoy1995scheduling} for batch scheduling on a single machine with incompatible job families, have been proposed for other batch scheduling problems.

The heuristic starts at time $t =0$. At every time step $t$, the list of currently available machines is generated: these are all machines on which no batch is running at time $t$ and for which $t$ lies within a machine availability interval.
For these available machines, the list of available jobs is generated:
these are jobs that have not yet been scheduled, have already been released and can be processed on one of the available machines. Among these jobs, the one with the earliest due date is chosen. If there are several jobs that fulfill these conditions, the largest one is chosen first. Let this job be $j \in \mathcal{J}$.
Moreover, if there are several machines for which $j$ fits into the current availability interval and that are eligible for $j$, the heuristic chooses the machine for which the setup time from the previous batch (or from the initial state of the machine) is minimal.

Once a job $j$ is scheduled, the algorithm tries to add other jobs that are currently available to the same batch:
For this, jobs are only considered if their attribute as well as minimal and maximal processing times are compatible with job $j$ and the combined batch size does not exceed the machine capacity.
Moreover, unless job $j$ finishes too late anyway, other jobs are only considered if their processing time does not force $j$ to finish late.
If there are several jobs that meet these requirements, we sort them in decreasing order of their latest end dates and keep adding jobs to the batch as long as the machine capacity and compatibility requirements are fulfilled.
If there are no jobs (left) that can be added to the batch and the maximum machine capacity is not reached yet, we look ahead and also consider jobs with an earliest start time $>t$. However, we only consider compatible jobs that do not force job $j$ to finish late (unless this was already the case) and that can still be processed within the current machine availability interval.
Once all jobs for the batch have been chosen, the batch is scheduled at the earliest possible time.

If there are other jobs that can be processed at time $t$, we schedule them and fill up their batches in the same way. If no job can be scheduled, the time is increased to $t+1$ and the above procedure is repeated until the end of the scheduling horizon is reached or all jobs have been scheduled.

\section{Experimental evaluation}\label{sec:experiments}

\subsection{Benchmark instances}

Using our random instance generator described in Section~\ref{sec:random-instances}, we created a large set of benchmark instances to evaluate the performance of our proposed models.
First, we executed the random generator once for every possible configuration of the 15 parameter values specified in Table~\ref{tab:param_random_instance_generator}.
Thereby we produced 1024 instances for each of the 16 combinations of the parameters $n$, $k$ and $a$.
Then we randomly selected 5 instances from every set of 1024 instances, creating a set of 80 instances which we used throughout our experiments. This set thus consists of 20 instances each with 10 (instances 1-20), 20 (21-40), 50 (41-60) and 100 jobs (61-80). All benchmark instances turn out to be satisfiable and solvable by the construction heuristic described in Section~\ref{sec:heuristic}. This reflects our real-life industrial application, for which feasible solutions usually can be found heuristically and the main aim is to find cost-minimal schedules.
Note that this set of benchmark instances is slightly different to the one presented in the authors' conference paper~\cite{lackner_et_al:LIPIcs.CP.2021.37}, since we have added initial states of machines.
Note also that the objective function as defined in equation~\eqref{eqn:obj-function-int} no longer includes setup times; we can therefore not expect that the optimal solutions for this benchmark set have the same solution cost as in the conference paper~\cite{lackner_et_al:LIPIcs.CP.2021.37}. 

\subsection{Experimental setup}

An overview of the models we evaluated on our benchmark set can be found in Table~\ref{tab:evaluated-models}:
We implemented the CP- and ILP-models presented in Sections~\ref{sec:cp-model}, \ref{sec:mip-model} and~\ref{sec:cp-repr}  using the high-level constraint modeling language MiniZinc~\cite{nethercote_minizinc_2007} and used recent versions %of the MIP solver Gurobi as well as the CP solver Chuffed 
of Chuffed, OR-Tools, CP Optimizer and Gurobi.
%to evaluate the performance of both models. %(Note that we also conducted experiments using Gurobi with the CP-model, as MiniZinc supports an automatic conversion into a ILP-model). 
For Chuffed and OR-Tools, we used all 7 search strategies described in Section~\ref{sec:search} and compared them with the solver's default search strategy.
%Additionally, we also conducted runs with just the solvers default search (ds).
For Chuffed, we activated the free search parameter which allows the solver to interleave between the given search strategy and its default search.
Furthermore, we investigated a warm-start approach with Gurobi (for the other solvers, warm-start is currently not supported by MiniZinc): the construction heuristic described in Section~\ref{sec:heuristic} was used to find an initial solution which was then provided to the model.
Finally, the OPL-models presented in Sections~\ref{sec:model-repr-job} and \ref{sec:model-opl-basic} were run using CP Optimizer in IBM ILOG CPLEX Studio.
This results in a total of 59 different combinations of models, solvers and search strategies per instance.
The time limit for every one of these combinations was set to one hour per instance; this includes the flattening time in MiniZinc and the model extraction time for CP Optimizer.
To configure the objective function for our experiments, we selected weights that model the importance of the individual components in typical practical applications together with our industry partners : $w_p=4$,  $w_{sc}=1$ and $w_t=100$. 

\begin{table}[]
    \centering
\begin{tabular}{m{2.5cm}m{2cm}m{2cm}m{2.5cm}}
    \toprule
    & CP model with batch positions & ILP model with batch positions & CP model with representative jobs\\ \midrule
      solver-independent MiniZinc model   & ``cp'' (Section~\ref{sec:cp-model}) & ``ilp'' (Section~\ref{sec:mip-model}) & ``cp-repr'' (Section~\ref{sec:cp-repr})\\ %\midrule
      MiniZinc model with warm-start  & ``cp-ws''  & ``ilp-ws''  & ``cp-repr-ws''\\ %\midrule
      OPL-model for CP Optimizer  & -- & ``opl-ilp'' (Section~\ref{sec:model-opl-basic}) & ``opl-cp-repr'' (Section~\ref{sec:model-repr-job})\\ \bottomrule
    \end{tabular}
    \caption{Overview of the experimentally evaluated models.}
    \label{tab:evaluated-models}
\end{table}

Experiments were run on single cores, using a computing cluster with 10 identical nodes, each having 24 cores, an Intel(R) Xeon(R) CPU E5--2650 v4 @ 2.20GHz and 252 GB RAM\@.

Throughout this section, we use abbreviations for the evaluated solvers and models. Solvers are referred to as ``chuffed'', ``gurobi'', ``cpopt'' and ``ortools''. The abbreviations for our models are summarized in Table~\ref{tab:evaluated-models}.

\subsection{Experimental results}

In the following we summarize our findings.
Table~\ref{tbl:overall_stats} provides an overview of the final results produced on the 80 benchmark instances with all evaluated methods.
Based on the experiments with different search strategies that are presented in Section~\ref{sec:results-search}, 
we selected the search strategy that lead to best results per solver and model.
This lead to the following search strategy-solver pairings in the overview of the experimental results: chuffed-cp with \emph{search6}, chuffed-ilp with \emph{search4}, chuffed-cp-repr with \emph{search3}, ortools-cp with \emph{search4}, ortools-ilp with \emph{search3}, ortools-cp-repr with \emph{search6}, and for all other solvers we used the \emph{default} search strategy.
%For methods where different search strategies were evaluated, the best search strategy is chosen per instance. 
The first column in each row denotes the evaluated solver and model. 
%First, we are interested in comparing the outcomes produced by each method. The outcome for an instance can be: unknown (=no solution found), unsatisfiable, solution found (might be optimal but no proof is provided) or solution with optimality proof found.
From left to right, the columns display: the number of solved instances, the number of instances where overall best, i.e. minimal, solution cost results could be achieved, the number of obtained optimal solutions, the number of optimality proofs, the number of instances where the fastest optimality proof could be found and the number of instances where the best, i.e. maximal, lower bound could be found. 

\begin{table}[ht]
  \setlength{\tabcolsep}{4pt}
  \centering
\begin{tabular}{lllllll}
\toprule
solver            & solved & best & optimal & proven opt & fastest proof & best lower bound \\ \midrule
gurobi-ilp        & 64     & 53   & 37      & 31         & 7             & 36               \\
gurobi-cp         & 46     & 35   & 33      & 23         & 0             & 24               \\
gurobi-cp-repr    & 33     & 27   & 27      & 25         & 0             & 25               \\
gurobi-ilp-ws     & 80     & 52   & 38      & 34         & 11            & 36               \\
gurobi-cp-ws      & 80     & 36   & 33      & 24         & 0             & 25               \\
gurobi-cp-repr-ws & 39     & 29   & 29      & 26         & 2             & 28               \\
cpopt-ilp         & 71     & 38   & 32      & 13         & 0             & 22               \\
cpopt-cp          & 67     & 33   & 30      & 13         & 0             & 26               \\
cpopt-cp-repr     & 48     & 30   & 29      & 18         & 0             & 50               \\
chuffed-ilp       & 70     & 24   & 24      & 14         & 0             & n/a              \\
chuffed-cp        & 67     & 24   & 24      & 14         & 0             & n/a              \\
chuffed-cp-repr   & 28     & 23   & 23      & 20         & 0             & n/a              \\
ortools-ilp       & 78     & 20   & 20      & 20         & 0             & 20               \\
ortools-cp        & 71     & 19   & 19      & 19         & 0             & 19               \\
ortools-cp-repr   & 56     & 21   & 21      & 18         & 3             & 18               \\
opl-ilp           & 49     & 20   & 20      & 17         & 0             & 17               \\
opl-cp-repr       & 80     & 66   & 37      & 28         & 15            & 28           \\\bottomrule   
\end{tabular}
\caption{Overview of the final computational results based on 80 benchmark instances.}\label{tbl:overall_stats}
\end{table}

%Second, we are interested in how fast optimality proofs can be obtained. 
For the subset of 9 instances  for which all solver-model combinations could deliver optimality proofs,  Table~\ref{tbl:overall_stats_solved_opt_by_all} contains further information (where available for the given solver): the average number of nodes visited in the search process, the average runtime, the standard deviation of the number of visited nodes, and the standard deviation of the runtime. 
These 9 instances all consist of 10 jobs.

\begin{table}[ht]
  \setlength{\tabcolsep}{4pt}
  \centering
  \begin{tabular}{lllll}
  \toprule
solver            & avg nodes & avg rt & std nodes & std rt \\\midrule
gurobi-ilp        & 3.62E+02  & 1.7    & 5.51E+02  & 1.7    \\
gurobi-cp         & 8.25E+02  & 24.4   & 8.58E+02  & 38.5   \\
gurobi-cp-repr    & 7.19E+02  & 37.5   & 9.38E+02  & 47.4   \\
gurobi-ilp-ws     & 3.27E+02  & 1.7    & 4.34E+02  & 1.5    \\
gurobi-cp-ws      & 4.62E+02  & 12.4   & 5.58E+02  & 14.7   \\
gurobi-cp-repr-ws & 6.55E+02  & 22.4   & 9.76E+02  & 26.1   \\
cpopt-ilp         & 1.13E+07  & 571.7  & 1.29E+07  & 712    \\
cpopt-cp          & 2.93E+06  & 130.4  & 3.63E+06  & 172.3  \\
cpopt-cp-repr     & 3.44E+06  & 511.2  & 5.67E+06  & 861.3  \\
chuffed-ilp       & 3.71E+05  & 106.6  & 7.63E+05  & 186.1  \\
chuffed-cp        & 2.95E+05  & 64     & 5.45E+05  & 105.5  \\
chuffed-cp-repr   & 6.06E+05  & 24.2   & 1.28E+06  & 36.8   \\
ortools-ilp       & n/a       & 28.3   & n/a       & 64.5   \\
ortools-cp        & n/a       & 15     & n/a       & 22.2   \\
ortools-cp-repr   & n/a       & 42     & n/a       & 106.2  \\
opl-ilp           & 1.33E+06  & 81.3   & 1.66E+06  & 165.7  \\
opl-cp-repr       & 3.49E+04  & 1.3    & 4.02E+04  & 1.7   \\\bottomrule
\end{tabular}
    \caption{Proof times and number of visited nodes for the subset of instances for which all solvers could prove optimality.}\label{tbl:overall_stats_solved_opt_by_all}
\end{table}

\subsubsection{Finding solutions}
%We can see in Table~\ref{tbl:overall_stats} that 
The OPL model with representative jobs (opl-cp-repr) finds solutions for all 80 instances.
Moreover, the warm-start approach for Gurobi (both mip-ws and cp-ws) finds solutions for all instances. %(see page \pageref{warm-start} for a detailed analysis of warm-start with Gurobi). 
Even without warm-start, ortools-ilp was capable of finding solutions for 78 instances (71 for ortools-cp), followed by cpopt-ilp (71 instances) and chuffed-ilp (70 instances).
For all solvers run via MiniZinc, both the ILP and the CP model lead to much better results than the CP model using representative batches.
However, when using CP Optimizer directly, significantly better results were achieved with opl-cp-repr compared with opl-ilp.

Regarding the quality of found solutions, best results were also achieved by opl-cp-repr (66 best results), followed by gurobi-ilp (53 best results) and gurobi-ilp-ws (52 best results). 
With cpopt-ilp and cpopt-cp 
as well as gurobi-cp-ws and gurobi-cp, best results could be achieved for roughly half of the instance set.
In comparison, Chuffed and OR-Tools produced fewer best instance results; best results could only be achieved for those instances that were solved optimally.
The results further indicate that concerning the ability of finding best solutions, Gurobi performs significantly better with the ILP-model. 
For CP Optimizer run via MiniZinc, the ILP-model also lead to the best results.
For Chuffed and  OR-Tools however, 
the results achieved with the ILP model are comparable with those achieved with one of the two CP models.

\subsubsection{Finding optimal solutions} Using all evaluated methods, optimal solutions could be found for 38 instances:  all instances with 10 jobs, 16 instances with 25 jobs, one with 50 jobs and one with 100 jobs. 

Most optimality proofs were provided by Gurobi and the ILP-model (34 proofs with warms-start, 31 without), followed by CP Optimizer run directly with the olp-repr model (28 proofs) and Gurobi with the CP models (between 23 and 26 proofs). 
Even though cpopt-cp-repr (cpopt-ilp, cpopt-cp)  could only provide 18 (13, 13) optimality proofs, it did find 29 (32, 30) optimal solutions; 
similarly chuffed-cp-repr (chuffed-cp, chuffed-ilp) provided only 20 (14, 14) optimality proofs but could find 23 (24, 24) optimal solutions.
For OR-Tools and the ILP-model (CP-model), optimality proofs could be delivered for all 20 (19) instances where the optimal solution was found.
It is noteworthy that for Chuffed and CP Optimizer most optimality proofs could be found when using the cp-repr model; for Gurobi more proofs could be found with cp-repr than with cp.

Figure~\ref{plot:proof_time_boxplot} takes a closer look at the 9 instances for which 
all evaluated methods provided provably optimal solutions within the runtime limit (instances 1-6, 9, 12 and 15) and compares the relative proof times for these instances. 
To calculate the relative proof time for an instance, we divide the absolute proof time (in seconds) by the overall best absolute proof time for that instance.
As Table~\ref{tbl:overall_stats_solved_opt_by_all} indicates, the average proof times for CP Optimizer used via MiniZinc exceed the average proof times of other solvers approximately by a factor of 10; we have therefore excluded the results for CP Optimizer run via MiniZinc from Figure~\ref{plot:proof_time_boxplot} to allow for better readability of the results. 
The distribution of relative proof times confirms what can already be observed Table~\ref{tbl:overall_stats_solved_opt_by_all}: fastest proofs are delivered by opl-cp-repr and by
gurobi-ilp (both with and without warm-start);
optimality proofs required the longest time for chuffed-cp and chuffed-ilp as well as opl-ilp.
Moreover, it can be noted that Gurobi and OR-Tools perform best with the ILP-model, whereas Chuffed and CP Optimizer run directly achieve best results with cp-repr respectively opl-cp-repr.
\begin{figure}
  \centering
  \includegraphics[width=.9\textwidth]{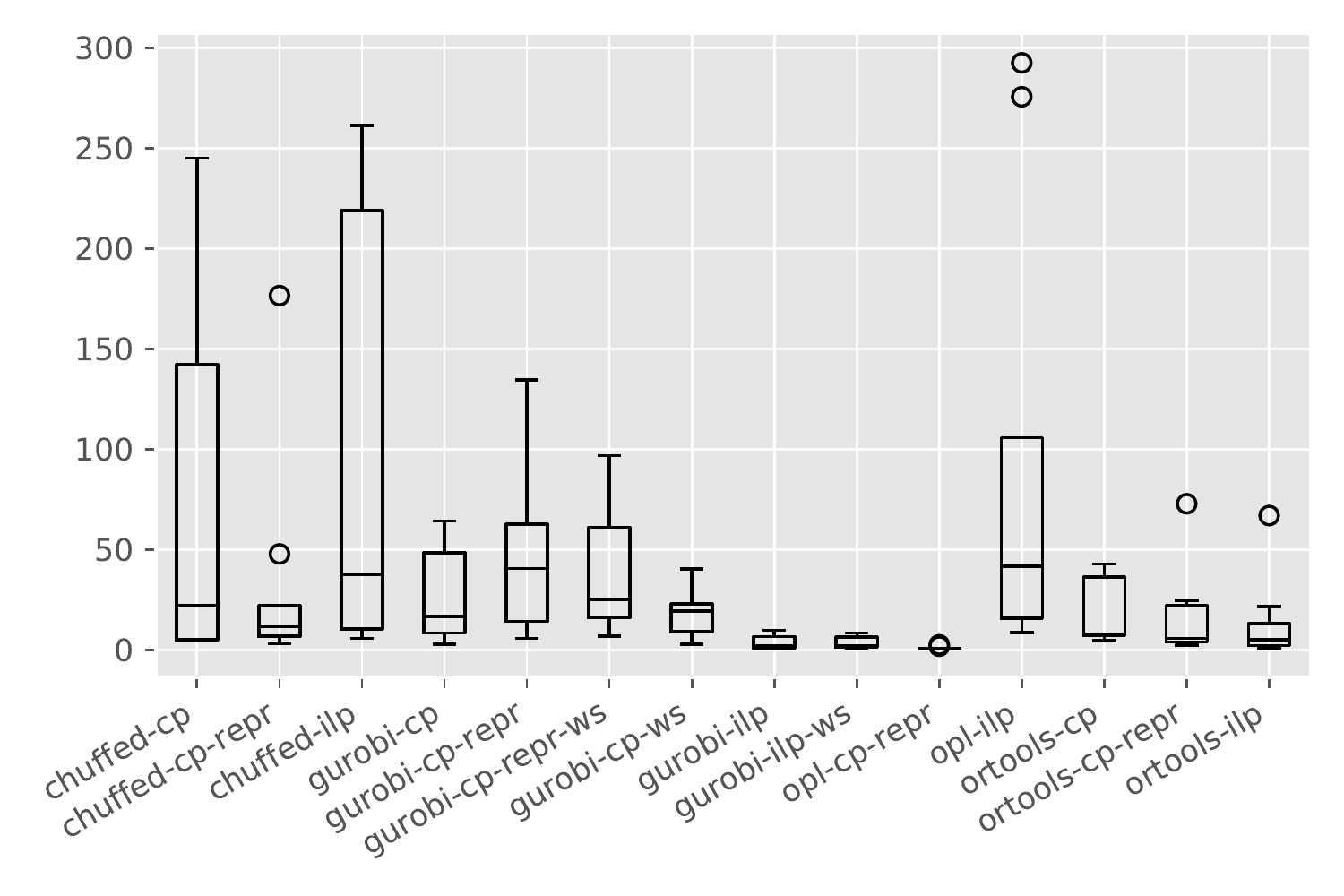}
  \captionof{figure}{Comparison of relative proof times for the subset of instances for which all methods could find optimality proofs.}\label{plot:proof_time_boxplot}
\end{figure}

\subsubsection{Search strategies.}
\label{sec:results-search}
The results of the comparison of search strategies for ortools-cp, ortools-cp-repr, and ortools-ilp can be found in Table~\ref{tbl:search_strategy_comparison_ortools}, those for chuffed-cp, chuffed-cp-repr and chuffed-ilp in Table~\ref{tbl:search_strategy_comparison_chuffed}. The first column in each row denotes the evaluated search strategy %for which we used the names introduced in Section~\ref{sec:search} 
and the following columns contain the respective numbers of solved instances, best solution results and optimality proofs achieved by each search strategy in comparison to all other search strategies with the same solver and model.
\begin{table}[ht]
  \setlength{\tabcolsep}{4pt}
  \centering
\begin{tabular}{l|lll|lll|lll}\toprule
search        & \multicolumn{3}{c|}{ortools-cp} & \multicolumn{3}{c|}{ortools-cp-repr} & \multicolumn{3}{c}{ortools-ilp} \\
strategy        & solved    & best    & proof    & solved      & best      & proof     & solved     & best    & proof    \\\midrule
default & 38        & 28      & 20       & 4           & 4         & 4         & 40         & 35      & 20       \\
search1 & 37        & 28      & 20       & 29          & 18        & 11        & 43         & 37      & 20       \\
search2 & 71        & 28      & 18       & 31          & 11        & 7         & 78         & 55      & 20       \\
search3 & 71        & 34      & 19       & 55          & 32        & 16        & 78         & 58      & 20       \\
search4 & 71        & 37      & 19       & 39          & 9         & 7         & 78         & 54      & 20       \\
search5 & 70        & 29      & 18       & 33          & 13        & 9         & 77         & 57      & 20       \\
search6 & 70        & 33      & 19       & 56          & 38        & 18        & 78         & 52      & 20       \\
search7 & 71        & 37      & 19       & 41          & 15        & 11        & 78         & 57      & 20      \\\bottomrule
\end{tabular}
\caption{Comparison of 8 search strategies for CP- and ILP-models with OR-Tools}\label{tbl:search_strategy_comparison_ortools}
\end{table}

\begin{table}[ht]
  \setlength{\tabcolsep}{4pt}
  \centering
\begin{tabular}{l|lll|lll|lll}\toprule
search       & \multicolumn{3}{c|}{chuffed-cp} & \multicolumn{3}{c|}{chuffed-cp-repr} & \multicolumn{3}{c}{chuffed-ilp} \\
strategy        & solved    & best    & proof    & solved      & best      & proof     & solved     & best    & proof    \\\midrule
default & 44        & 20      & 8        & 20          & 14        & 14        & 41         & 18      & 9        \\
search1 & 65        & 28      & 14       & 23          & 22        & 20        & 62         & 23      & 14       \\
search2 & 66        & 29      & 14       & 26          & 23        & 20        & 65         & 27      & 14       \\
search3 & 66        & 31      & 14       & 28          & 26        & 20        & 65         & 29      & 14       \\
search4 & 65        & 31      & 14       & 26          & 22        & 20        & 70         & 33      & 14       \\
search5 & 65        & 31      & 14       & 26          & 23        & 20        & 64         & 31      & 14       \\
search6 & 67        & 38      & 14       & 27          & 25        & 21        & 63         & 31      & 14       \\
search7 & 66        & 32      & 14       & 26          & 22        & 20        & 68         & 36      & 14     \\\bottomrule
\end{tabular}
\caption{Comparison of 8 search strategies for CP- and ILP-models with Chuffed}\label{tbl:search_strategy_comparison_chuffed}
\end{table}

We can see that using search strategies could greatly improve the number of solved instances for all six compared methods.
The improvement was most significant for ortools-cp-repr, which could only solve 4 instances with the default strategy and 56 instances with search strategy 6.
Moreover, the number of instances solved could nearly by doubled for ortools-ilp when using search strategies 2 to 7 (40 with the default strategy and 78 with search strategies 2--4 and 6,7) and increased from 38 to 70 or 71 for ortools-cp when using one of the search strategies  1 to 7.
For Chuffed, all search strategies had a comparable impact on the number of solved instances.  
The number of best solutions found could be improved as well for all six methods; again, the most notable improvement was for ortools-cp-repr and search strategy 6 (4 vs. 38 best solutions). For chuffed-cp (chuffed-ilp) search strategy 6 (7) nearly allowed to  double the number of best solutions found from 20 to 38 (18 to 36).
For Chuffed, the number of provided optimality proofs could also greatly be improved using the suggested search strategies. 
For OR-Tools, the chosen search strategies do not necessarily have a positive impact on the number of optimality proof found. For ortools-cp, using search strategies 2--7 even lead to slightly fewer optimality https://www.overleaf.com/project/623b1a528f09888d1fbc5b52proofs. However, for ortools-cp-repr, large improvements could again be made using strategies 3 and 6. 

To summarize, for all methods except ortools-cp-repr, the 6 search strategies 2--7 had a comparable impact on the number of solved instances and on the solution quality. For Chuffed, improvements could also be achieved by using search strategy 1, which tries to assign the number of batches first and then continues with the solver's default search strategy.
For ortools-cp-repr, the differences between search strategies are greater: search strategies 3 and 6 allow for the greatest improvement.

Search strategies can also be used for Gurobi and are implemented with
branching priorities by the MiniZinc interface.
However, the experiments run for gurobi-ilp and gurobi-cp with all search strategies did not reveal significant differences between the search strategies for this solver.
For CP Optimizer run via MiniZinc, search strategies are currently not supported.

\subsubsection{Warm-start with Gurobi.}\label{warm-start}
The construction heuristic could find feasible solutions for all 80 instances within a few seconds. 
On average, the construction heuristic required 0.271 seconds per instance, 57 out of 80 instances could be solved in less than 0.1 seconds and the maximum time required per instance was 5.6 seconds.
Warm-starting gurobi-cp-ws and gurobi-ilp-ws with these initial solutions 
%(solvers gurobi-ilp-ws and gurobi-cp-ws) 
%thus allowed us to 
lead to solutions for all 80 instances, i.e, 34 more instances than with gurobi-cp and 16 more than with gurobi-ilp.
For gurobi-cp-repr-ws could however still only solve 39 instances (6 more than without the warm-start approach).\footnote{The warmstart data provided to Gurobi only contains values for a subset of the decision variables. The solver thus needs to complete the partial solution and, for gurobi-cp-repr-ws, fails to do so for 41 instances.}
%, including large scale instances. 
%  Figure~\ref{fig:sol_quality_graph_warm_start} allows a comparison of the solution quality per instance for the construction heuristic, gurobi-ilp and gurobi-ilp-ws.
% \begin{figure}[ht]
%   \centering
%   \includegraphics[width=1\textwidth]{plots/sol_quality_warmstart.pdf}
%   \caption{Solution quality per instance using the construction heuristic, gurobi-ilp without warm-start and gurobi-ilp-ws with the heuristically constructed solution for warm-start.}\label{fig:sol_quality_graph_warm_start}
% \end{figure}

The heuristic solution could be improved by gurobi-ilp-ws for  78 instances:
for instance number 8 the construction heuristic interestingly already finds an optimal solution and for instance number 73, no improvement could be made (the solution is however not optimal, as shown by a solution with much lower solutions cost found by oplrun-repr).
For 8 instances for which gurobi-ilp had previously found solutions, warm-starting could improve the solution quality. 
For 7 instances warm-starting however led to a lower solution quality.
%For the remaining 47 instances (of which ... are optimal), the results from gurobi-ilp remained unchanged.
Overall, the relative difference between the solution quality of gurobi-ilp and gurobi-ws was not very large: for a single instance, the relative difference was greater than 5\%.
Concerning optimality proofs, gurobi-ilp-ws was  slightly worse than gurobi-ilp (32 vs. 36 proofs).
Similar results were obtained with gurobi-cp-ws.
To sum up, the warm-start approach for Gurobi can mainly help to find solutions for instances that Gurobi could otherwise not solve; the impact on improving the solution quality is however limited.

\subsubsection{Upper and lower bounds on the solution cost}

\paragraph{Best lower bounds.}
Lower bounds (also referred to as dual bounds) on the solution cost of an optimal solution are provided by Gurobi, OR-Tools and CP Optimizer (both when called directly and via MiniZinc). 
Best lower bounds were found by CP-Optimier run via MiniZinc and the cp-repr approach (best bounds for 50 instances), followed by gurobi-ilp (36 instances).
We compare the best lower bounds with those obtained with the theoretical results obtained in Section~\ref{sec:lower-bounds}.
Note that the bounds obtained by Gurobi, OR-Tools and CP Optimizer where provided within the time limit of one hour.
For the calculation of the theoretical lower bounds from Section~\ref{sec:lower-bounds}, our straightforward implementation yielded results in less than 2 seconds for all instances in our benchmark set.

\begin{figure}[ht]
  \centering
  \includegraphics[width=\textwidth]{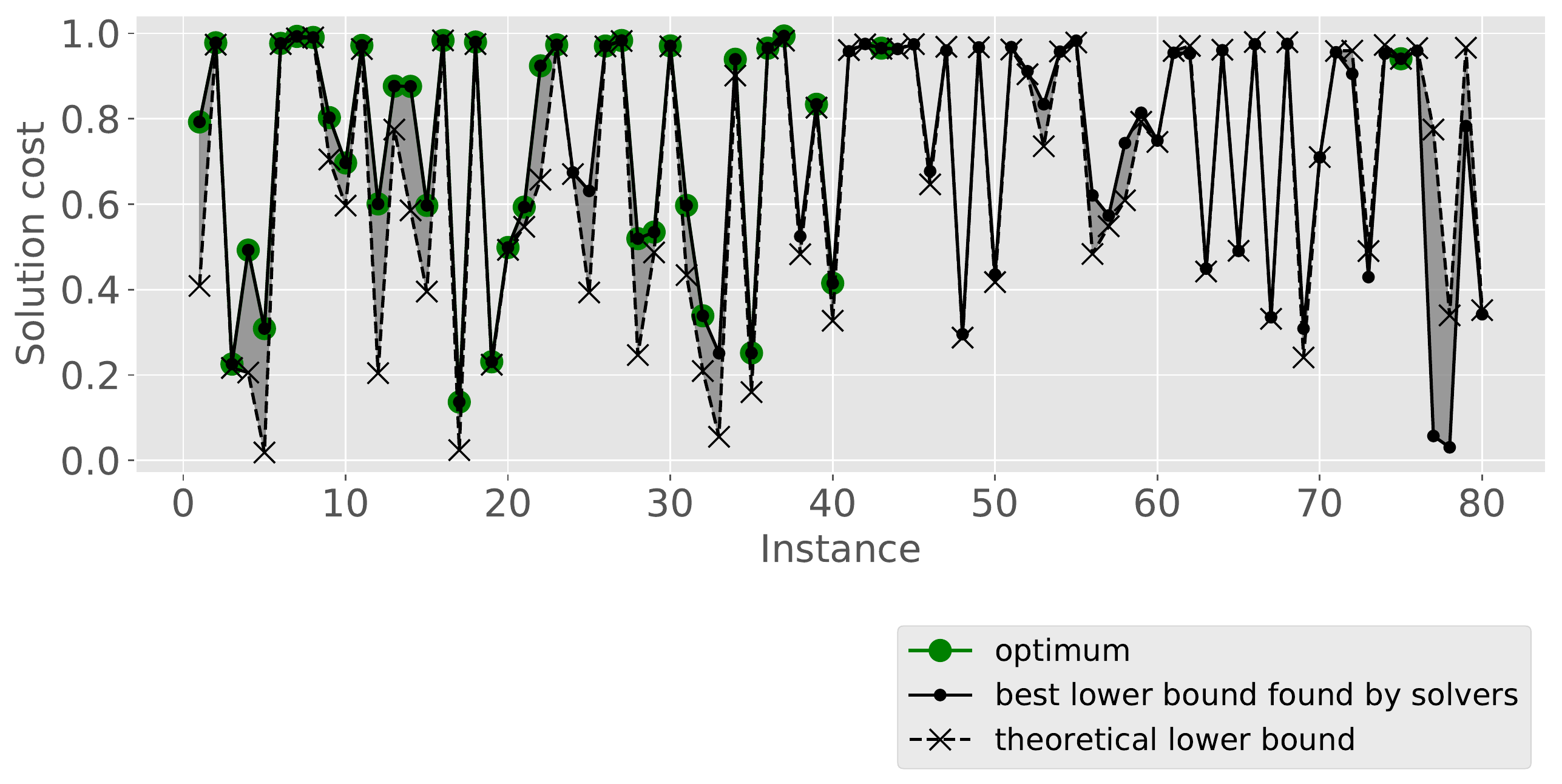}
  \caption{Best lower bounds found by solvers compared with the theoretical lower bounds.}\label{fig:best_lower_bounds_vs_calc_lower_bounds}
\end{figure}

As we can see in Figure~\ref{fig:best_lower_bounds_vs_calc_lower_bounds}, the
theoretical results produce bounds that are competitive with the lower bounds obtained by solvers during the process of solving the OSP.
The dashed line in this figure displays the theoretical lower bounds on the value of the objective function for every one of the benchmark instances.
The value of the aggregated objective function $\text{obj}$ as defined in equation~\eqref{eqn:obj-function}  is obtained by computing the lower bounds on the cumulative batch processing time via equation~\eqref{eqn:best-lower-bound-proc-time}, the bounds on the cumulative setup costs via equation~\eqref{eqn:lower-bound-setup-costs} and bounds on the number of tardy jobs as explained at the end of Section~\ref{sec:lower-bounds-obj-components}.
The solid line displays the overall best, i.e. maximal, lower bound found by Gurobi, OR-Tools or CP Optimizer when using any one of the evaluated models with any of the search strategies for every instance of the benchmark set.
For instances that could be solved optimally by some solver-model combination, the best lower bound coincides with the best upper bound and is additionally marked by a green circle.
A first observation is that the theoretical lower bounds are correlated with the best lower bounds obtained by the evaluated solvers. Moreover, the gap between these bounds, marked by the gray surfaces in Figure~\ref{fig:best_lower_bounds_vs_calc_lower_bounds}, is seldom very large.

This observation can be quantified by calculating the relative gap $g(i) = \vert \text{calc}(i) - \text{best}(i) \vert/\text{best}(i)$, where $\text{calc}(i)$ is the calculated lower bound and $\text{best}(i)$ the best lower bound found by solvers for instance $i$.
For 35 instances, the gap $g(i)$ is less than 1\% and for 50 instances it is less than 5\%.
Among the 45 instances for which the gap is larger than 1\%, the calculated bound is better than the best bound found by solvers, i.e. $g(i) > 0.01$, for 8 instances, all of which have 100 jobs.
% Note also that among these 45 instances, 23 instances could be solved optimally by some solver.
In total, among the 20 instances with 100 jobs (instances 61--80), the calculated bounds were better than the best bounds found by some solver in 16 cases. 

\paragraph{Optimality gap.}

In order to assess the quality of solutions found for instances where the optimum was not found, we compute the optimality gap for every instance.
That is, if $s(i)$ is the objective value of the overall best solution found (i.e., the minimal solution cost found by any solver-model combination) and $b(i)$ is the best (i.e. maximal) lower bound found (either by some solver or by the theoretical lower bounds) for instance $i$, the optimality gap is given by $g(i) = (s(i)-b(i))/s(i)$.
Figure~\ref{fig:optimality_gap} visualizes the overall smallest optimality gap per instance. 
\begin{figure}[ht]
  \centering
  \includegraphics[width=\textwidth]{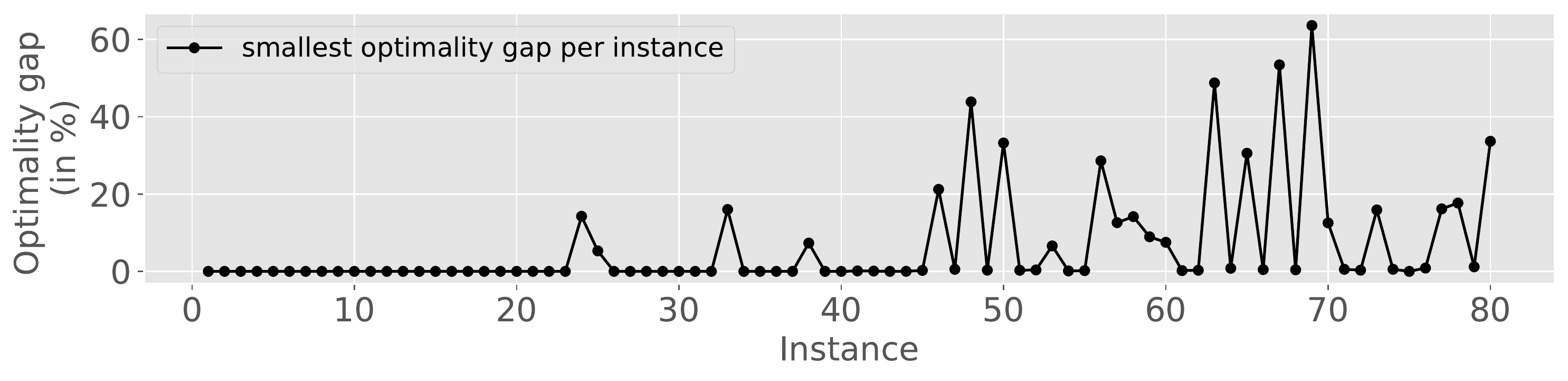}
  \caption{Overall smallest optimality gap per instance.}\label{fig:optimality_gap}
\end{figure}
The optimality gap generally increases with the number of jobs per instance: while the lower bounds are tight for all instances with 10 jobs and for most instances with 25 jobs (optimal solutions were found for 16 out of 20 instances with 25 jobs), this is no longer the case for instances with 50 or 100 jobs. 
However, the size of the optimality gap is not purely determined by the number of jobs;
there are 21 instances with 50 jobs or more for which the optimality gap is less than 1\% (2 of these instances could actually be solved optimally).
In total, the optimality gap is less than 1\% for 57 instances and smaller than  10\% for 63 instances out of 80.

\paragraph{Bounds on the number of batches and the components of the objective function.}

We take a closer look at the upper and lower bounds derived for the three components of the objective function as well as the number of batches used.
For the number of batches, the cumulative batch processing time,  the cumulative setup costs and the number of tardy jobs we compare the calculated lower bounds with the values for the overall best solution cost.\footnote{
Note that we cannot compare the overall best values with the overall best lower bounds, since the lower bounds derived by solvers cannot be broken down into components of the objective function.}
In order to make these values comparable across all instances, we calculated the following normalized measures for the lower and upper bounds for every instance:
\begin{itemize}
    \item The number of batches required per job: $b / n \in [1/n, 1]$,
    where $b$ is the total number of batches used.
    \item The batch processing time divided by the maximal batch processing time: $r = p/ (\sum_{j \in{J}} mint_j) \in (0,1]$, where $p$ is the cumulative batch processing time and $\sum_{j \in{J}} mint_j$ is the cumulative processing time if every job is processed in an individual batch. That is, the processing time achieved with the help of batching is equal to $r \cdot \sum_{j \in{J}} mint_j$; if no batching was done, $r=1$. The value of $r$ can thus be seen as the factor by which the processing time could be reduced with the help of batching.
    \item The setup costs divided by the maximal setup costs: $st / (n \cdot \max(1,max_{ST}))$, where  $max_{ST}$ is the maximal setup cost for this instance.
    \item The percentage of tardy jobs: $t/n$.
\end{itemize}

\begin{figure}
     \centering
     \begin{subfigure}[b]{0.9\textwidth}
         \centering
         \includegraphics[width=\textwidth]{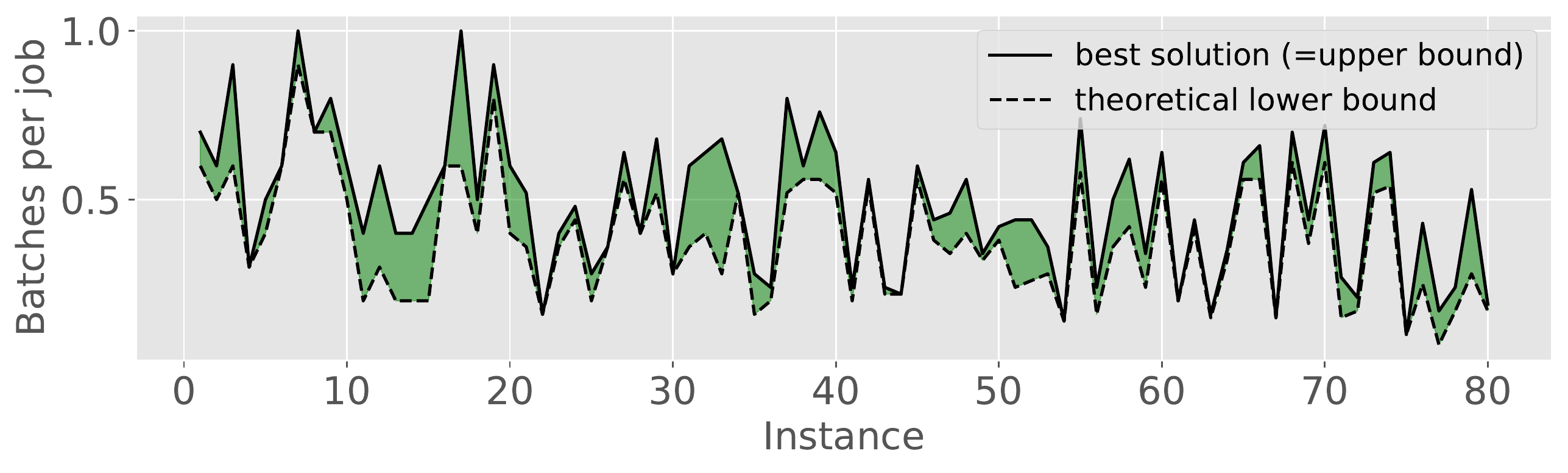}
         \caption{Bounds for the number of batches required per job}
         \label{fig:bounds-number-of-batches}
     \end{subfigure}
     \hfill
     \begin{subfigure}[b]{0.9\textwidth}
         \centering
         \includegraphics[width=\textwidth]{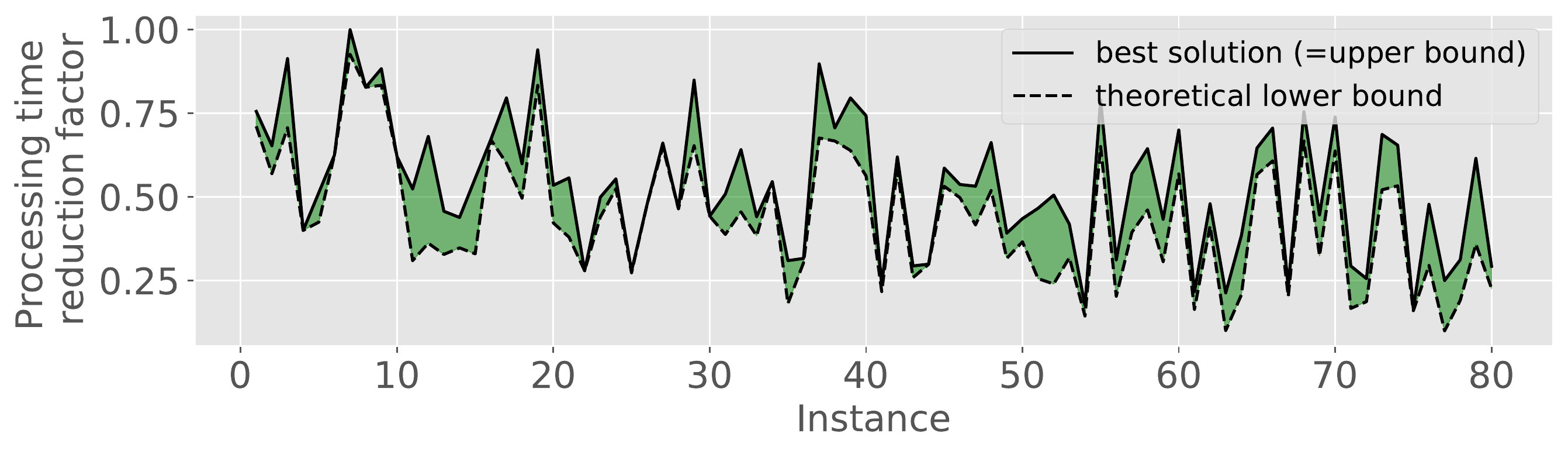}
         \caption{Bounds for the processing time reduction that is possible with batching}
         \label{fig:bounds-components_proc-time}
     \end{subfigure}
     \hfill
     \begin{subfigure}[b]{0.9\textwidth}
         \centering
         \includegraphics[width=\textwidth]{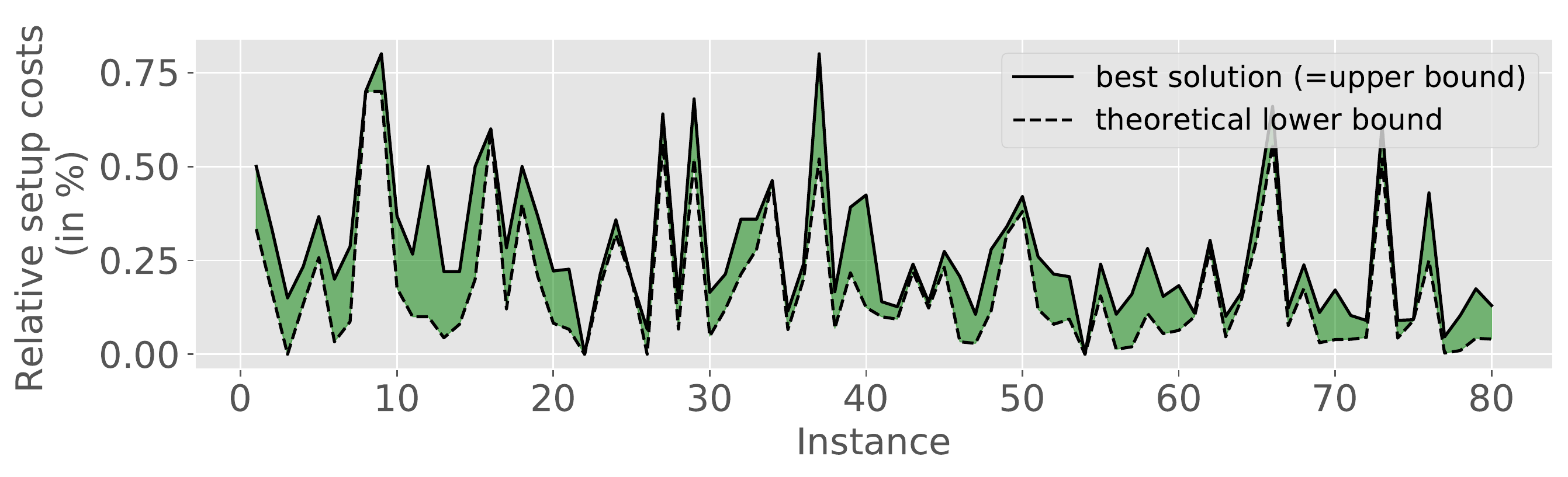}
         \caption{Bounds on the relative setup costs}
         \label{fig:bounds-components_setupcosts}
     \end{subfigure}
     \hfill
     \begin{subfigure}[b]{0.9\textwidth}
         \centering
         \includegraphics[width=\textwidth]{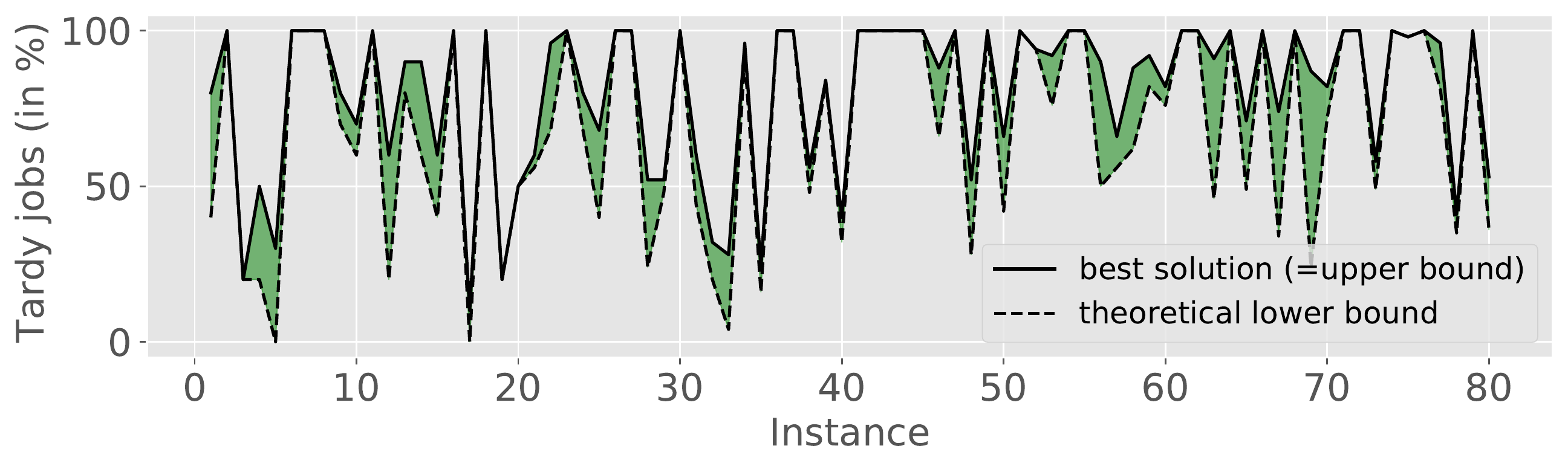}
         \caption{Bounds on the percentage of tardy jobs}
         \label{fig:bounds-components_tardy_jobs}
     \end{subfigure}
        \caption{Comparison of upper and lower bounds on the number of batches and the components of the objective function per instance.}
        \label{fig:bounds-components}
\end{figure}

The results are displayed in Figure~\ref{fig:bounds-components}: Figure~\ref{fig:bounds-number-of-batches} shows upper and lower bounds for the number of batches required, Figure~\ref{fig:bounds-components_proc-time} for the cumulative processing times, Figure~\ref{fig:bounds-components_setupcosts} for the cumulative setup costs and Figure~\ref{fig:bounds-components_tardy_jobs} for the tardiness of jobs.
These plots show that for all four measures, the calculated lower bounds are close to the best upper bounds for many instances. The calculation of the relative gap between upper and lower bounds allows to quantify this observation. The average gap and the standard deviation of gaps for our benchmark set are recorded in the table below:

\medskip
\begin{center}
    \begin{tabular}{lcc}
    \toprule
    & avg & std \\ \midrule
      number of batches   & $20.5\%$ &  $16.8\%$\\
      batch processing time  & $19.4\%$ & $14.8\%$\\
      setup costs & $43.9\%$ & $28.3\%$\\
      tardy jobs & $17.3\%$ & $24.3\%$ \\\bottomrule
    \end{tabular}
\end{center}
\medskip

The calculated lower bounds for setup costs leave the greatest room for improvement, as can be seen from the largest average and standard deviation. This is not surprising, since the bounds for setup costs derived in Section~\ref{sec:lower-bounds-obj-components} are quite simple in comparison with the refined lower bounds for the number of batches and the batch processing time presented in Section~\ref{sec:lower-bounds-batch-number}. 
Moreover, compared with the gaps for the number of batches and the batch processing times, the gap for the number of tardy jobs shows a higher standard deviation. This implies that even if the lower bounds are relatively good on average, they are quite far off for some instances. This can also be seen in Figure~\ref{fig:bounds-components_tardy_jobs}: upper and lower bounds coincide for 39 instances (explaining the low average gap) but are also very far apart for many instances, the gap being larger than 25\% for 23 instances.

\paragraph{Construction heuristic vs. calculated lower bounds.}

In analogy to the optimality gap per instance visualised in Figure~\ref{fig:optimality_gap}, we compute the relative gap between the calculated lower bound and the cost of the solution produced by the construction heuristic for every instance.
Note that if the construction heuristic is successful in scheduling all jobs, which it is for all instances of our benchmark set, the produced solution is always feasible and the obtained solution cost is thus an upper bound on the optimal solution cost.
The results are shown in Figure~\ref{fig:gap_greedy_calcLB}.
\begin{figure}[ht]
  \centering
  \includegraphics[width=\textwidth]{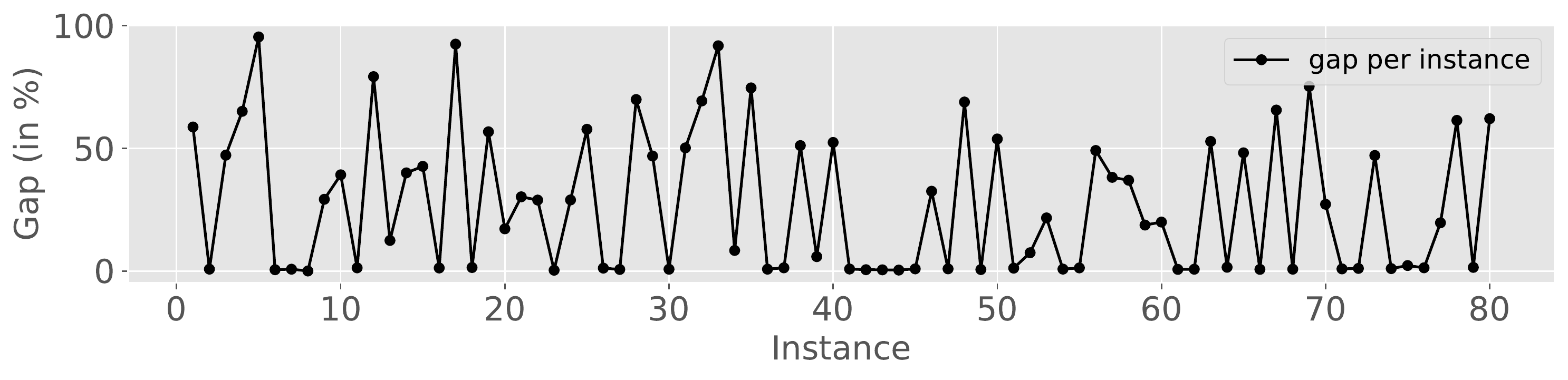}
  \caption{Gap between the upper bound found by the construction heuristic and the calculated lower bound per instance.}\label{fig:gap_greedy_calcLB}
\end{figure}
Obviously, the results are not as good as those for the optimality gap and the gap is nearly equal to 100\% for a few instances.
However, it is quite surprising to note that the gap is less than 1\% for 22 instances of all sizes, including 13 instances for which no solver could provably find an optimal solution. 
For a total of 37 instances, including 20 instances for which no solver could provably find an optimal solution, the gap is  less than 10\%. 
This means that, within a couple of seconds of computation time (the construction heuristic finds a solution for every instance within less than 6 seconds), we can find good estimates on the optimal solution cost for a quarter of the instance set and rough estimates for almost half of the instance set.

\section{Conclusion}
In this paper, we introduced and formally defined the Oven Scheduling Problem and provided instances for this problem, a new batch scheduling problem that appears in the electronic component industry. 
We proposed two different modelling approaches, presented corresponding CP- and ILP-models and investigated various search strategies. 

Using our models as well as a warm-start approach, we were able to find feasible solutions for all 80 benchmark instances. 
Provably optimal solutions could be found for nearly half of the instance set and for 57 of the 80 instances, the best optimality gap is less than 1\%. 
Overall, the best results could be achieved with the model using representative jobs (see Section~\ref{sec:model-repr-job}) and running CP Optimizer directly, followed by the ILP-model with warm-start and Gurobi (see Section~\ref{sec:mip-model}). OR-Tools with the ILP model was also capable of finding solutions for almost all instances.
Varying the search strategy had a major impact on the performance of the CP solvers Chuffed and OR-Tools.

Furthermore, we developed refined lower bounds on the number of batches required and the cumulative batch processing time in any feasible solution which allows to provide lower bounds on the optimal solution cost within a couple of seconds.
These theoretical lower bounds turn out to be competitive with the best lower bounds found by the evaluated solvers for the majority of the benchmark instances.
For large instances, the theoretical lower bounds are particularly good: for 16 of 20 instances with 100 jobs, the theoretical lower bounds are better than those found by any of the evaluated solvers within one hour.
Moreover, using these lower bounds together with the upper bounds on the optimal solution cost which can be obtained by using a simple construction heuristic, an interval for the optimum cost value can be computed in a few seconds. 
For 22 of the benchmark instances, including 13 instances for which no solver could provably find an optimal solution within one hour, the relative gap between these upper and lower bounds is less than 1\%.

To further improve the solution quality for large instances, we plan to develop meta-heuristic strategies based on local search or large neighborhood search.
Moreover, in order to explain which parameters cause instances to be hard, an in-depth instance space analysis could be conducted. This could also shed light onto which methods perform best on different parts of the instance space.

\subsection*{Acknowledgments}

The financial support by the Austrian Federal Ministry for Digital and Economic Affairs, the National Foundation for Research, Technology and Development and the Christian Doppler Research Association is gratefully acknowledged.

% ---- Bibliography ----
\bibliography{osp_journal}

\end{document}